\documentclass{article}

\pdfoutput=1

\usepackage{amssymb,amsmath,amsthm}
\usepackage[nolist]{acronym}
\usepackage{algorithm,algorithmic}
\usepackage{enumerate}
\usepackage{fullpage}
\usepackage{booktabs}       
\usepackage{nicefrac}       
\usepackage{graphicx}
\usepackage{url}
\usepackage{subfigure}
\usepackage{natbib}

\newtheorem{definition}{Definition}
\newtheorem{lemma}[definition]{Lemma}
\newtheorem{theorem}[definition]{Theorem}
\newtheorem{corollary}[definition]{Corollary}

\DeclareMathOperator{\Regret}{Regret}
\DeclareMathOperator{\Wealth}{Wealth}
\DeclareMathOperator{\Reward}{Reward}
\DeclareMathOperator*{\argmin}{arg\,min}
\DeclareMathOperator*{\argmax}{arg\,max}

\newcommand{\R}{\mathbb{R}}     
\renewcommand{\H}{\mathcal{H}}  
\newcommand{\KL}[2]{\operatorname{D}\left({#1}\middle\|{#2}\right)}  
\newcommand{\norm}[1]{\left\|{#1}\right\|}
\newcommand{\indicator}{\mathbf{1}}
\newcommand{\scO}{\mathcal{O}}  

\begin{acronym}
\acro{EG}{Exponentiated Gradient}
\acro{DFEG}{Dimension-Free Exponentiated Gradient}
\acro{OMD}{Online Mirror Descent}
\acro{ASGD}{Averaged Stochastic Gradient Descent}
\acro{SGD}{Stochastic Gradient Descent}
\acro{PiSTOL}{Parameter-free STOchastic Learning}
\acro{OCO}{Online Convex Optimization}
\acro{OLO}{Online Linear Optimization}
\acro{RKHS}{Reproducing Kernel Hilbert Space}
\acro{IID}{Independent and Identically Distributed}
\acro{SVM}{Support Vector Machine}
\acro{ERM}{Empirical Risk Minimization}
\acro{COCOB}{Continous Coin Betting}
\acro{MBA}{Master Betting Algorithm}
\acro{KT}{Krichevsky-Trofimov}
\acro{LEA}{Learning with Expert Advice}
\acro{OGD}{Online Gradient Descent}
\acro{KL}{Kullback-Leibler}
\end{acronym}

\author{
  Francesco Orabona\\
  Stony Brook University, Stony Brook, NY \\
  \texttt{francesco@orabona.com}
  \and
  D\'avid P\'al\\
  Yahoo Research, New York, NY\\
  \texttt{dpal@yahoo-inc.com}
}

\title{Coin Betting and Parameter-Free Online Learning}

\begin{document}

\maketitle

\begin{abstract}
In the recent years, a number of parameter-free algorithms have been developed
for online linear optimization over Hilbert spaces and for learning with expert
advice.  These algorithms achieve optimal regret bounds that depend on the
unknown competitors, without having to tune the learning rates with oracle
choices.

We present a new intuitive framework to design parameter-free algorithms for \emph{both} online linear optimization over Hilbert spaces and for learning with expert
advice, based on reductions to betting on outcomes of adversarial coins. We instantiate it
using a betting algorithm based on the Krichevsky-Trofimov estimator.  The
resulting algorithms are simple, with no parameters to be tuned, and they
improve or match previous results in terms of regret guarantee and per-round
complexity.
\end{abstract}

\vspace{-0.2cm}

\section{Introduction}
\label{section:introduction}

\vspace{-0.2cm}

We consider the \ac{OLO}~\cite{Cesa-Bianchi-Lugosi-2006, Shalev-Shwartz-2011}
setting. In each round $t$, an algorithm chooses a point $w_t$ from a convex
\emph{decision set} $K$ and then receives a reward vector $g_t$. The
algorithm's goal is to keep its \emph{regret} small, defined as the difference
between its cumulative reward and the cumulative reward of a fixed strategy $u
\in K$, that is
\vspace{-.1cm}
\[
\Regret_T(u) = \sum_{t=1}^T \langle g_t, u \rangle - \sum_{t=1}^T \langle g_t, w_t \rangle \; .
\]
We focus on two particular decision sets, the $N$-dimensional probability
simplex $\Delta_N = \{ x \in \R^N ~:~ x \ge 0, \norm{x}_1 = 1\}$ and a
Hilbert space $\H$.  \ac{OLO} over $\Delta_N$ is referred to as the problem of
\ac{LEA}.  We assume bounds on the norms of the reward vectors: For \ac{OLO}
over $\H$, we assume that $\norm{g_t} \le 1$, and for \ac{LEA} we assume that
$g_t \in [0,1]^N$.

\vspace{-0.1cm}

\ac{OLO} is a basic building block of many machine learning problems. For
example, \ac{OCO}, the problem analogous to \ac{OLO} where $\langle g_t, u
\rangle$ is generalized to an arbitrary convex function $\ell_t(u)$, is solved
through a reduction to \ac{OLO}~\cite{Shalev-Shwartz-2011}.
\ac{LEA}~\cite{Littlestone-Warmuth-1994, Vovk-1998,
Cesa-Bianchi-Freund-Haussler-Helmbold-Schapire-Warmuth-1997} provides a way of
combining classifiers and it is at the heart of
boosting~\cite{Freund-Schapire-1997}. Batch and stochastic convex optimization
can also be solved through a reduction to \ac{OLO}~\cite{Shalev-Shwartz-2011}.

\vspace{-0.1cm}

To achieve optimal regret, most of the existing online algorithms require the
user to set the learning rate (step size) $\eta$ to an unknown/oracle
value.  For example, to obtain the optimal bound for \ac{OGD}, the learning
rate has to be set with the knowledge of the norm of the competitor $u$,
$\norm{u}$; second entry in Table~\ref{table:bounds}.  Likewise, the optimal
learning rate for Hedge depends on the KL divergence between the prior
weighting $\pi$ and the unknown competitor $u$, $\KL{u}{\pi}$; seventh entry in
Table~\ref{table:bounds}.  Recently, new parameter-free algorithms have been
proposed, both for \ac{LEA}~\cite{Chaudhuri-Freund-Hsu-2009, Chernov-Vovk-2010,
Luo-Schapire-2014, Luo-Schapire-2015, Koolen-van-Erven-2015,
Foster-Rakhlin-Sridharan-2015} and for \ac{OLO}/\ac{OCO} over Hilbert
spaces~\cite{Streeter-McMahan-2012, Orabona-2013, McMahan-Abernethy-2013,
McMahan-Orabona-2014, Orabona-2014}.  These algorithms adapt to the number of
experts and to the norm of the optimal predictor, respectively, without the
need to tune parameters. However, their \emph{design and underlying intuition}
is still a challenge.  \citet{Foster-Rakhlin-Sridharan-2015} proposed a unified
framework, but it is not constructive.  Furthermore, all existing algorithms
for LEA either have sub-optimal regret bound (e.g. extra $\scO(\log \log T)$
factor) or sub-optimal running time (e.g.  requiring solving a numerical
problem in every round, or with extra factors); see Table~\ref{table:bounds}.

\begin{table}
\centering
\resizebox{\linewidth}{!}{%
\begin{tabular}{l c c c c}
\toprule
Algorithm & Worst-case regret guarantee& \begin{tabular}{@{}c@{}}Per-round time\\complexity\end{tabular} & Adaptive & \begin{tabular}{@{}c@{}}Unified \\  analysis\end{tabular}\\
\midrule
OGD, $\eta=\tfrac{1}{\sqrt{T}}$ \cite{Shalev-Shwartz-2011} & $\scO((1 + \norm{u}^2)\sqrt{T})$, $\forall u \in \H$ & $\scO(1)$ &  \\
OGD, $\eta=\tfrac{U}{\sqrt{T}}$ \cite{Shalev-Shwartz-2011} & $U \sqrt{T}$ for any $u \in \H$ s.t. $\norm{u} \le U$ & $\scO(1)$ &  \\
\cite{Orabona-2013} & $\scO(\norm{u} \ln(1+\norm{u}T) \sqrt{T})$, $\forall u \in \H$ & $\scO(1)$ & \checkmark \\
\cite{McMahan-Orabona-2014,Orabona-2014} & $\scO(\norm{u}\sqrt{T \ln(1+\norm{u}T)})$, $\forall u \in \H$ & $\scO(1)$ & \checkmark \\
This paper, Sec.~\ref{section:kt-olo} & $\scO(\norm{u}\sqrt{T \ln(1+\norm{u}T)})$, $\forall u \in \H$ & $\scO(1)$ & \checkmark & \checkmark\\
\midrule
Hedge, $\eta=\sqrt{\tfrac{\ln N}{T}}$, $\pi_i=\tfrac{1}{N}$~\cite{Freund-Schapire-1997} & $\scO(\sqrt{T \ln N})$, $\forall u \in \Delta_N$ & $\scO(N)$ &  \\
Hedge, $\eta=\tfrac{U}{\sqrt{T}}$~\cite{Freund-Schapire-1997} & $\scO(U\sqrt{T})$ for any $u \in \Delta_N$ s.t. $\sqrt{\KL{u}{\pi}} \le U$ & $\scO(N)$ &  \\
\cite{Chaudhuri-Freund-Hsu-2009}  & $\scO(\sqrt{T (1+\KL{u}{\pi})}+\ln^2 N)$, $\forall u \in \Delta_N$ & $\scO(N\,K)$\footnotemark[1]& \checkmark \\
\cite{Chernov-Vovk-2010} & $\scO(\sqrt{T \left(1+\KL{u}{\pi}\right)})$, $\forall u \in \Delta_N$ & $\scO(N\,K)$\footnotemark[1] & \checkmark \\
\cite{Chernov-Vovk-2010, Luo-Schapire-2015,Koolen-van-Erven-2015}\footnotemark[2] & $\scO(\sqrt{T \left(\ln \ln T+\KL{u}{\pi}\right)})$, $\forall u \in \Delta_N$ & $\scO(N)$ & \checkmark \\
\cite{Foster-Rakhlin-Sridharan-2015} & $\scO(\sqrt{T \left(1+\KL{u}{\pi}\right)})$, $\forall u \in \Delta_N$ & $\scO(N \ln \max_{u \in \Delta_N} \KL{u}{\pi})$\footnotemark[3] & \checkmark & \checkmark \\
This paper, Sec.~\ref{section:kt-lea} & $\scO(\sqrt{T \left(1+\KL{u}{\pi}\right)})$, $\forall u \in \Delta_N$ & $\scO(N)$ & \checkmark & \checkmark\\
\bottomrule
\end{tabular}}
\caption{\footnotesize{Algorithms for \ac{OLO} over Hilbert space and \ac{LEA}.
}}
\label{table:bounds}
\vspace{-0.7cm}
\end{table}
\footnotetext[1]{These algorithms require to solve a numerical problem at each step. The number $K$ is the number of steps needed to reach the required precision. Neither the precision nor $K$ are calculated in these papers.
}
\footnotetext[2]{The proof in \cite{Koolen-van-Erven-2015} can be modified to prove a KL bound, see \url{http://blog.wouterkoolen.info}.}
\footnotetext[3]{A variant of the algorithm in \cite{Foster-Rakhlin-Sridharan-2015} can be implemented with the stated time complexity~\cite{Foster:private}.}

\textbf{Contributions.} We show that a more fundamental notion subsumes
\emph{both} \ac{OLO} and \ac{LEA} parameter-free algorithms. We prove that the
ability to maximize the wealth in bets on the outcomes of coin flips
\emph{implies} \ac{OLO} and \ac{LEA} parameter-free algorithms. We develop a
novel potential-based framework for betting algorithms. It gives intuition to
previous constructions and, instantiated with the Krichevsky-Trofimov
estimator, provides new and elegant algorithms for \ac{OLO} and \ac{LEA}.  The
new algorithms also have optimal worst-case guarantees on regret and time
complexity; see Table~\ref{table:bounds}.

\section{Preliminaries}
\label{section:preliminaries}

We begin by providing some definitions.  The \ac{KL} divergence between two
discrete distributions $p$ and $q$ is $\KL{p}{q} = \sum_{i} p_i \ln\left(
p_i/q_i \right)$. If $p,q$ are real numbers in $[0,1]$, we denote by $\KL{p}{q} =
p \ln \left(p/q \right) + (1-p) \ln \left((1-p)/(1-q) \right)$ the \ac{KL}
divergence between two Bernoulli distributions with parameters $p$ and $q$.  We
denote by $\H$ a Hilbert space, by $\langle \cdot, \cdot\rangle$ its inner
product, and by $\norm{\cdot}$ the induced norm.  We denote by $\norm{\cdot}_1$
the $1$-norm in $\R^N$.  A function $F:I \to \R_+$ is called
\emph{logarithmically convex} iff $f(x) = \ln(F(x))$ is convex.  Let $f:V \to \R
\cup \{\pm\infty\}$, the Fenchel conjugate of $f$ is $f^*:V^* \to \R \cup \{\pm
\infty\}$ defined on the dual vector space $V^*$ by $f^*(\theta) = \sup_{x \in
V} \ \langle \theta, x \rangle - f(x)$.  A function $f:V \to \R \cup
\{+\infty\}$ is said to be \emph{proper} if there exists $x \in V$ such that
$f(x)$ is finite.  If $f$ is a proper lower semi-continuous convex function then
$f^*$ is also proper lower semi-continuous convex and $f^{**}=f$.

\textbf{Coin Betting.} We consider a gambler making repeated bets on the
outcomes of adversarial coin flips. The gambler starts with an initial
endowment $\epsilon > 0$. In each round $t$, he bets on the outcome of a coin
flip $g_t \in \{-1,1\}$, where $+1$ denotes heads and $-1$ denotes tails.  We
do not make any assumption on how $g_t$ is generated, that is, it can be chosen
by an adversary.

The gambler can bet any amount on either heads or tails. However, he is not
allowed to borrow any additional money. If he loses, he loses the betted
amount; if he wins, he gets the betted amount back and, in addition to that, he
gets the same amount as a reward.  We encode the gambler's bet in round $t$ by
a single number $w_t$. The sign of $w_t$ encodes whether he is betting on heads
or tails. The absolute value encodes the betted amount.  We define $\Wealth_t$
as the gambler's wealth at the end of round $t$ and $\Reward_t$ as the
gambler's net reward (the difference of wealth and initial endowment), that is
\begin{align}
\label{equation:def_wealth_reward}
\Wealth_t = \epsilon + \sum_{i=1}^t w_i g_i &
& \text{and} &
& \Reward_t = \Wealth_t - \ \epsilon \; .
\end{align}
In the following, we will also refer to a bet with $\beta_t$, where $\beta_t$
is such that
\begin{equation}
\label{equation:def_wt}
w_t = \beta_t \Wealth_{t-1} \; .
\end{equation}
The absolute value of $\beta_t$ is the \emph{fraction} of the current wealth to
bet, and sign of $\beta_t$ encodes whether he is betting on heads or tails. The
constraint that the gambler cannot borrow money implies that $\beta_t \in
[-1,1]$.
We also generalize the problem slightly by allowing the outcome of the coin flip
$g_t$ to be any real number in the interval $[-1,1]$; wealth and reward in~\eqref{equation:def_wealth_reward} remain exactly
the same.

\section{Warm-Up: From Betting to One-Dimensional Online Linear Optimization}
\label{section:one-dimensional-hilbert-space-olo}

In this section, we sketch how to reduce one-dimensional OLO to betting on a
coin. The reasoning for generic Hilbert spaces
(Section~\ref{section:reduction_hilbert}) and for LEA
(Section~\ref{section:reduction-experts}) will be similar. We will show that the
betting view provides a natural way for the analysis and design of online
learning algorithms, where the only design choice is the potential function of
the betting algorithm (Section~\ref{section:coin-betting-potentials}). A
specific example of coin betting potential and the resulting algorithms are in
Section~\ref{section:kt-estimator}.

As a warm-up, let us consider an algorithm for OLO over one-dimensional Hilbert
space $\R$.  Let $\{w_t\}_{t=1}^\infty$ be its sequence of predictions on a
sequence of rewards $\{g_t\}_{t=1}^\infty$, $g_t \in [-1,1]$. The total reward
of the algorithm after $t$ rounds is $\Reward_t = \sum_{i=1}^t g_i w_i$. Also,
even if in OLO there is no concept of ``wealth'', define the wealth of the OLO
algorithm as $\Wealth_t = \epsilon + \Reward_t$, as in
\eqref{equation:def_wealth_reward}.

We now restrict our attention to algorithms whose predictions $w_t$ are of the
form of a bet, that is $w_t = \beta_t \Wealth_{t-1}$, where $\beta_t \in
[-1,1]$.  We will see that the restriction on $\beta_t$ does not prevent us
from obtaining parameter-free algorithms with optimal bounds.

Given the above, it is immediate to see that any coin betting algorithm that,
on a sequence of coin flips $\{g_t\}_{t=1}^\infty$, $g_t \in [-1,1]$, bets the
amounts $w_t$ can be used as an OLO algorithm in a one-dimensional Hilbert
space $\R$. But, what would be the regret of such OLO algorithms?

Assume that the betting algorithm at hand guarantees that its wealth is at least
$F(\sum_{t=1}^T g_t)$ starting from an endowment $\epsilon$, for a given
potential function $F$, then
\vspace{-.1cm}
\begin{equation}
\label{equation:one-dimensional-olo-reward-lower-bound}
\Reward_T
= \sum_{t=1}^T g_t w_t
= \Wealth_T \ - \ \epsilon \ge F\left(\sum_{t=1}^T g_t \right) \ - \ \epsilon \; .
\end{equation}
Intuitively, if the reward is big  we can expect the regret to be small. Indeed,
the following lemma converts the lower bound on the reward to an upper bound on
the regret.
\begin{lemma}[Reward-Regret relationship~\cite{McMahan-Orabona-2014}]
\label{lemma:reward-regret}
Let $V,V^*$ be a pair of dual vector spaces. Let $F:V \to \R \cup \{+\infty\}$
be a proper convex lower semi-continuous function and let $F^*:V^* \to \R \cup
\{+\infty\}$ be its Fenchel conjugate. Let $w_1, w_2, \dots, w_T \in V$ and
$g_1, g_2, \dots, g_T \in V^*$. Let $\epsilon \in \R$. Then,
\[
\underbrace{\sum_{t=1}^T \langle g_t, w_t \rangle}_{\Reward_T} \ge F\left( \sum_{t=1}^T g_t \right) -\epsilon
\qquad \text{if and only if} \qquad
\forall u \in V^*, \quad
\underbrace{\sum_{t=1}^T \langle g_t, u - w_t\rangle}_{\Regret_T(u)} \le F^*(u) + \epsilon\; .
\]
\end{lemma}
\vspace{-.1cm}
Applying the lemma, we get a regret upper bound:
$\Regret_T(u) \le F^*(u) + \epsilon$ for all $u \in \H$.

To summarize, if we have a betting algorithm that guarantees a minimum wealth
of $F(\sum_{t=1}^T g_t)$, it can be used to design and analyze a
one-dimensional \ac{OLO} algorithm. The faster the growth of the wealth, the
smaller the regret will be.  Moreover, the lemma also shows that trying to
design an algorithm that is adaptive to $u$ is \emph{equivalent} to designing
an algorithm that is adaptive to $\sum_{t=1}^T g_t$.  Also, most importantly,
\emph{methods that guarantee optimal wealth for the betting scenario are
already known}, see, e.g., \cite[Chapter 9]{Cesa-Bianchi-Lugosi-2006}. We can
just re-use them to get optimal online algorithms!

\section{Designing a Betting Algorithm: Coin Betting Potentials}
\label{section:coin-betting-potentials}


For sequential betting on i.i.d. coin flips, an optimal strategy has been
proposed by \citet{Kelly-1956}.  The strategy assumes that the coin flips
$\{g_t\}_{t=1}^\infty$, $g_t \in \{+1,-1\}$, are generated i.i.d. with known
probability of heads. If $p \in [0,1]$ is the probability of heads, the Kelly
bet is to bet $\beta_t = 2p - 1$ at each round. He showed that, in the long
run, this strategy will provide more wealth than betting any other fixed
fraction of the current wealth~\cite{Kelly-1956}.

For adversarial coins, Kelly betting does not make sense. With perfect
knowledge of the future, the gambler could always bet everything on the right
outcome.  Hence, after $T$ rounds from an initial endowment $\epsilon$, the
maximum wealth he would get is $\epsilon 2^T$.  Instead, assume he bets the
same fraction $\beta$ of its wealth at each round.  Let $\Wealth_t(\beta)$ the
wealth of such strategy after $t$ rounds.  As observed in
\cite{McMahan-Abernethy-2013}, the optimal fixed fraction to bet is
$\beta^*=(\sum_{t=1}^T g_t)/T$ and it gives the wealth
\begin{equation}
\label{eq:opt_wealth}
\Wealth_T(\beta^*)
= \epsilon \exp\left(T \cdot \KL{\tfrac{1}{2}+\tfrac{\sum_{t=1}^T g_t}{2T}}{\tfrac{1}{2}}\right)
\ge \epsilon \exp\left(\tfrac{(\sum_{t=1}^T g_t)^2}{2 T}\right) \; ,
\end{equation}
where the inequality follows from Pinsker's inequality~\citep[Lemma
11.6.1]{Cover-Thomas-2006}.

However, even without knowledge of the future, it is possible to go very close
to the wealth in \eqref{eq:opt_wealth}.  This problem was studied by
\citet{Krichevsky-Trofimov-1981}, who proposed that after seeing the coin flips
$g_1, g_2, \dots, g_{t-1}$ the empirical estimate $k_t = \frac{1/2 +
\sum_{i=1}^{t-1} \indicator[g_i = +1]}{t}$ should be used instead of $p$. Their
estimate is commonly called \emph{KT estimator}.\footnote{Compared to the
maximum likelihood estimate $\frac{\sum_{i=1}^{t-1} \indicator[g_i =
+1]}{t-1}$, KT estimator shrinks slightly towards $\nicefrac{1}{2}$.} The KT
estimator results in the betting
\begin{equation}
\label{equation:kt-estimator-betting-strategy}
\beta_t = 2k_t - 1 = \tfrac{\sum_{i=1}^{t-1} g_i}{t}
\end{equation}
which we call \emph{adaptive Kelly betting based on the KT estimator}. It looks
like an online and slightly biased version of the oracle choice of $\beta^*$.
This strategy guarantees\footnote{See Appendix~\ref{section:logloss-to-wealth}
for a proof. For lack of space, all the appendices are in the supplementary
material.}
\[
\Wealth_T
\ge \tfrac{\Wealth_T(\beta^*)}{2\sqrt{T}}
= \tfrac{\epsilon}{2\sqrt{T}} \exp\left(T \cdot \KL{\tfrac{1}{2}+\tfrac{\sum_{t=1}^T g_t}{2T}}{\tfrac{1}{2}}\right)\; .
\]
This guarantee is optimal up to constant
factors~\citep{Cesa-Bianchi-Lugosi-2006} and mirrors the guarantee of the Kelly
bet.

Here, we propose a new set of definitions that allows to generalize the
strategy of adaptive Kelly betting based on the KT estimator. For these strategies
it will be possible to prove that, for any $g_1, g_2, \dots, g_t \in [-1,1]$,
\vspace{-0.3cm}
\begin{equation}
\label{equation:wealth-lower-bound-generic}
\Wealth_t \ge F_t \left( \sum_{i=1}^t g_i \right) \; ,
\end{equation}
where $F_t(x)$ is a certain function. We call such functions \emph{potentials}.
The betting strategy will be determined uniquely by the potential (see (c) in
the Definition~\ref{definition:potential}), and we restrict our attention to
potentials for which \eqref{equation:wealth-lower-bound-generic} holds. These
constraints are specified in the definition below.
\begin{definition}[Coin Betting Potential]
\label{definition:potential}
Let $\epsilon > 0$. Let $\{F_t\}_{t=0}^\infty$ be a sequence of functions
$F_t:(-a_t, a_t)  \to \R_+$ where $a_t > t$.  The sequence
$\{F_t\}_{t=0}^\infty$ is called a \textbf{sequence of coin betting potentials
for initial endowment $\epsilon$}, if it satisfies the following three
conditions:
\begin{enumerate}[(a)]
\item $F_0(0) = \epsilon$.

\item For every $t \ge 0$, $F_t(x)$ is even, logarithmically convex, strictly
increasing on $[0,a_t)$, and
$\lim_{x \to a_t} F_t(x) = +\infty$.
\item For every $t \ge 1$, every $x \in [-(t-1), (t-1)]$ and every $g \in [-1,1]$, $\left(1 + g \beta_t \right) F_{t-1}(x) \ge F_t(x+g)$, where
\begin{equation}
\label{equation:potential-based-strategy}
\beta_t=\tfrac{F_t(x + 1) - F_t(x - 1)}{F_t(x + 1) + F_t(x - 1)} \;.
\end{equation}
\end{enumerate}
The sequence $\{F_t\}_{t=0}^\infty$ is called a
\textbf{sequence of excellent coin betting potentials for initial
endowment $\epsilon$} if it satisfies conditions (a)--(c) and the condition (d)
below.
\begin{enumerate}[(a)]
\setcounter{enumi}{3}
\item For every $t \ge 0$, $F_t$ is twice-differentiable and
satisfies $x \cdot F_t''(x) \ge F_t'(x)$ for every $x \in [0,a_t)$.
\end{enumerate}
\end{definition}
Let's give some intuition on this definition.  First, let's show by induction
on $t$ that (b) and (c) of the definition together with \eqref{equation:def_wt} give a betting strategy that satisfies
\eqref{equation:wealth-lower-bound-generic}. The base case $t=0$ is trivial. At
time $t \ge 1$, bet $w_t=\beta_t \Wealth_{t-1}$
where $\beta_t$ is defined in \eqref{equation:potential-based-strategy}, then
\begin{align*}
\Wealth_t
&= \Wealth_{t-1} + w_t g_t
= (1+g_t \beta_t) \Wealth_{t-1} \\
&\ge (1 + g_t \beta_t) F_{t-1} \left(\sum_{i=1}^{t-1} g_i \right)
\ge F_t \left(\sum_{i=1}^{t-1} g_i + g_t \right)
= F_t \left( \sum_{i=1}^t g_i \right) \; .
\end{align*}
The formula for the potential-based
strategy~\eqref{equation:potential-based-strategy} might seem strange. However,
it is derived---see Theorem~\ref{theorem:optimal-betting-fraction}
in Appendix~\ref{section:optimal-betting-fraction}---by minimizing the
worst-case value of the right-hand side of the inequality used w.r.t. to $g_t$
in the induction proof above: $F_{t-1}(x) \ge \tfrac{F_{t}(x +
g_t)}{1+g_t\beta_t}$.

The last point, (d), is a technical condition that allows us to seamlessly
reduce OLO over a Hilbert space to the one-dimensional problem, characterizing
the worst case direction for the reward vectors.

Regarding the design of coin betting potentials, we expect any potential that
approximates the best possible wealth in \eqref{eq:opt_wealth} to be a good
candidate.  In fact, $F_t(x)=\epsilon \exp \left(x^2/(2t)\right)/\sqrt{t}$,
essentially the potential used in the parameter-free algorithms in
\cite{McMahan-Orabona-2014, Orabona-2014} for \ac{OLO} and in
\cite{Chaudhuri-Freund-Hsu-2009, Luo-Schapire-2014, Luo-Schapire-2015} for
\ac{LEA}, approximates \eqref{eq:opt_wealth} and it is an excellent coin
betting potential---see Theorem~\ref{thm:exp_x2} in
Appendix~\ref{section:optimal-betting-fraction}. Hence, our framework provides
intuition to previous constructions and in Section~\ref{section:kt-estimator}
we show new examples of coin betting potentials.

In the next two sections, we presents the reductions to effortlessly solve \emph{both} the generic \ac{OLO} case and \ac{LEA} with a betting potential.

\section{From Coin Betting to OLO over Hilbert Space}
\label{section:reduction_hilbert}

In this section, generalizing the one-dimensional construction in
Section~\ref{section:one-dimensional-hilbert-space-olo}, we show how to use a
sequence of excellent coin betting potentials $\{F_t\}_{t=0}^\infty$ to
construct an algorithm for \ac{OLO} over a Hilbert space and how to prove a
regret bound for it.

We define reward and wealth analogously to the one-dimensional case:
$\Reward_t = \sum_{i=1}^t \langle g_i, w_i \rangle$ and $\Wealth_t = \epsilon +
\Reward_t$. Given a sequence of coin betting potentials $\{F_t\}_{t=0}^\infty$,
using \eqref{equation:potential-based-strategy} we define the fraction
\begin{equation}
\label{equation:potential-based-strategy-hilbert-space}
\beta_t
= \tfrac{F_t \left(\norm{\sum_{i=1}^{t-1} g_i} + 1\right) - F_t\left(\norm{\sum_{i=1}^{t-1} g_i} - 1 \right)}{F_t\left(\norm{\sum_{i=1}^{t-1} g_i} + 1 \right) + F_t\left(\norm{\sum_{i=1}^{t-1} g_i} - 1 \right)} \; .
\end{equation}
The prediction of the OLO algorithm is defined similarly to the one-dimensional case, but now we also need a direction in the Hilbert space:
\begin{equation}
\label{equation:hilbert-space-olo}
w_t
= \beta_t \Wealth_{t-1} \frac{\sum_{i=1}^{t-1} g_i}{\norm{\sum_{i=1}^{t-1} g_i}}
= \beta_t \frac{\sum_{i=1}^{t-1} g_i}{\norm{\sum_{i=1}^{t-1} g_i}} \left(\epsilon+ \sum_{i=1}^{t-1} \langle g_i, w_i\rangle \right) \; .
\end{equation}
\vspace{-0.01cm}
If $\sum_{i=1}^{t-1} g_i$ is the zero vector, we define $w_t$ to be the zero vector
as well.  For this prediction strategy we can prove the following regret
guarantee, proved in Appendix~\ref{section:hilbert-space-reduction}.  The proof
reduces the general Hilbert case to the 1-d case, thanks to (d) in
Definition~\ref{definition:potential}, then it follows the reasoning of
Section~\ref{section:one-dimensional-hilbert-space-olo}.
\begin{theorem}[Regret Bound for OLO in Hilbert Spaces]
\label{theorem:hilbert-space-olo-regret-bound}
Let $\{F_t\}_{t=0}^\infty$ be a sequence of excellent coin betting potentials.
Let $\{g_t\}_{t=1}^\infty$ be any sequence of reward vectors in a Hilbert space
$\H$ such that $\norm{g_t} \le 1$ for all $t$. Then, the algorithm that makes
prediction $w_t$ defined by \eqref{equation:hilbert-space-olo} and
\eqref{equation:potential-based-strategy-hilbert-space} satisfies
\[
\forall T \ge 0 \quad
\forall u \in \H \qquad \qquad
\Regret_T(u) \le F_T^*\left(\norm{u} \right) \ + \ \epsilon \; .
\]
\end{theorem}

\section{From Coin Betting to Learning with Expert Advice}
\label{section:reduction-experts}

In this section, we show how to use the algorithm for OLO over one-dimensional
Hilbert space $\R$ from
Section~\ref{section:one-dimensional-hilbert-space-olo}---which is itself based
on a coin betting strategy---to construct an algorithm for \ac{LEA}.

\vspace{-0.05cm}

Let $N \ge 2$ be the number of experts and $\Delta_N$ be the $N$-dimensional
probability simplex. Let $\pi = (\pi_1, \pi_2, \dots, \pi_N) \in \Delta_N$ be
any \emph{prior} distribution. Let $A$ be an algorithm for OLO over
the one-dimensional Hilbert space $\R$, based on a sequence of the coin betting
potentials $\{F_t\}_{t=0}^\infty$ with initial endowment\footnote{Any initial
endowment $\epsilon > 0$ can be rescaled to $1$. Instead of $F_t(x)$ we would
use $F_t(x)/\epsilon$. The $w_t$ would become $w_t/\epsilon$, but $p_t$ is
invariant to scaling of $w_t$. Hence, the LEA algorithm is the same regardless
of $\epsilon$.} $1$. We instantiate $N$ copies of $A$.

\vspace{-0.05cm}

Consider any round $t$. Let $w_{t,i} \in \R$ be the prediction of the $i$-th copy of
$A$. The LEA algorithm computes $\widehat p_t = (\widehat p_{t,1}, \widehat
p_{t,2}, \dots, \widehat p_{t,N}) \in \R_{0,+}^N$ as
\begin{equation}
\label{eq:phat}
\widehat p_{t,i} = \pi_i \cdot [w_{t,i}]_+,
\end{equation}
where $[x]_+ = \max\{0,x\}$ is the positive part of $x$. Then, the LEA
algorithm predicts $p_t = (p_{t,1}, p_{t,2}, \dots, p_{t,N}) \in \Delta^N$ as
\begin{equation}
\label{eq:preds_experts}
p_t = \tfrac{\widehat p_t}{\norm{\widehat p_t}_1} \; .
\end{equation}
If $\norm{\widehat p_t}_1 = 0$, the algorithm predicts the prior $\pi$.
Then, the algorithm receives the reward vector
$g_t = (g_{t,1}, g_{t,2}, \dots, g_{t,N}) \in [0,1]^N$. Finally, it
feeds the reward to each copy of $A$. The reward for the $i$-th copy of $A$ is $\widetilde g_{t,i} \in
[-1,1]$ defined as
\begin{align}
\label{eq:gradients_experts_reduction}
\widetilde g_{t,i} =
\begin{cases}
g_{t,i} - \langle g_t, p_t \rangle & \text{if } w_{t,i} > 0 \; , \\
\left[g_{t,i} - \langle g_t, p_t \rangle \right]_+ & \text{if } w_{t,i} \le 0 \; .
\end{cases}
\end{align}

The construction above defines a \ac{LEA} algorithm defined by the predictions
$p_t$, based on the algorithm $A$.  We can prove the following regret bound for
it.
\begin{theorem}[Regret Bound for Experts]
\label{theorem:regret-bound-experts}
Let $A$ be an algorithm for \ac{OLO} over the one-dimensional Hilbert space
$\R$, based on the coin betting potentials $\{F_t\}_{t=0}^\infty$ for an
initial endowment of $1$. Let $f_t^{-1}$ be the inverse of $f_t(x) =
\ln(F_t(x))$ restricted to $[0,\infty)$.  Then, the regret of the \ac{LEA}
algorithm with prior $\pi \in \Delta_N$ that predicts at each round with $p_t$
in \eqref{eq:preds_experts} satisfies
\[
\forall T \ge 0 \quad \forall u \in \Delta_N \qquad \qquad
\Regret_T(u) \le f_T^{-1}\left( \KL{u}{\pi} \right) \; .
\]
\end{theorem}
The proof, in Appendix~\ref{section:appendix-expert-reduction}, is based on the
fact that \eqref{eq:phat}--\eqref{eq:gradients_experts_reduction} guarantee
that $\sum_{i=1}^N \pi_i \widetilde g_{t,i} w_{t,i} \le 0$ and on a variation
of the change of measure lemma used in the PAC-Bayes literature,
e.g.~\cite{McAllester-2013}.

\section{Applications of the Krichevsky-Trofimov Estimator to OLO and LEA}
\label{section:kt-estimator}

In the previous sections, we have shown that a coin betting potential with a
guaranteed rapid growth of the wealth will give good regret guarantees for
\ac{OLO} and \ac{LEA}. Here, we show that the KT estimator has associated an
excellent coin betting potential, which we call \emph{KT potential}.  Then, the
optimal wealth guarantee of the KT potentials will translate to optimal
parameter-free regret bounds.

The sequence of excellent coin betting potentials for an initial endowment $\epsilon$
corresponding to the adaptive Kelly betting strategy
$\beta_t$ defined by \eqref{equation:kt-estimator-betting-strategy}
based on the KT estimator are
\begin{equation}
\label{equation:kt-estimator-potential}
F_t(x) = \epsilon \tfrac{2^t \cdot \Gamma \left( \tfrac{t+1}{2} + \frac{x}{2} \right) \cdot \Gamma \left( \tfrac{t+1}{2} - \frac{x}{2} \right)}{\pi \cdot t!}
\qquad \qquad \text{$t \ge 0$, \quad $x \in \left(-t-1, t+1\right)$,}
\end{equation}
where $\Gamma(x) = \int_0^\infty t^{x-1} e^{-t} dt$ is Euler's gamma
function---see Theorem~\ref{theorem:kt-potential} in
Appendix~\ref{section:properties-kt-potential}.  This potential was used to
prove regret bounds for online prediction with the logarithmic
loss~\cite{Krichevsky-Trofimov-1981}\cite[Chapter
9.7]{Cesa-Bianchi-Lugosi-2006}.  Theorem~\ref{theorem:kt-potential} also shows
that the KT betting strategy $\beta_t$ as defined by
\eqref{equation:kt-estimator-betting-strategy} satisfies
\eqref{equation:potential-based-strategy}.

This potential has the nice property that is satisfies the inequality in
(c) of Definition~\ref{definition:potential} with equality when $g_t\in
\{-1,1\}$, i.e. $F_t(x+g_t)=(1+g_t \beta_t) \, F_{t-1}(x)$.

We also generalize the KT potentials to \emph{$\delta$-shifted KT
potentials}, where $\delta\geq0$, defined as
\[
F_t(x) = \tfrac{2^t \cdot \Gamma\left(\delta + 1 \right) \cdot \Gamma\left(\tfrac{t+\delta+1}{2} + \frac{x}{2} \right) \cdot \Gamma\left(\tfrac{t+\delta+1}{2} - \frac{x}{2} \right)}{\Gamma\left(\tfrac{\delta+1}{2} \right)^2 \cdot \Gamma \left(t+\delta+1\right)} \; .
\]
The reason for its name is that, up to a multiplicative constant, $F_t$ is
equal to the KT potential shifted in time by $\delta$.
Theorem~\ref{theorem:kt-potential} also proves that the $\delta$-shifted KT
potentials are excellent coin betting potentials with initial endowment $1$,
and the corresponding betting fraction is $\beta_t = \tfrac{\sum_{j=1}^{t-1}
g_j}{\delta+t}$.

\vspace{-0.1cm}

\subsection{OLO in Hilbert Space}
\label{section:kt-olo}

We apply the KT potential for the construction of an OLO algorithm over a
Hilbert space $\H$. We will use \eqref{equation:hilbert-space-olo}, and we just
need to calculate $\beta_t$. According to Theorem~\ref{theorem:kt-potential} in
Appendix \ref{section:properties-kt-potential}, the formula for $\beta_t$
simplifies to $\beta_t = \frac{\norm{\sum_{i=1}^{t-1} g_i}}{t}$ so that $w_t =
\tfrac{1}{t} \left(\epsilon + \sum_{i=1}^{t-1} \langle g_i, w_i \rangle \right)
\sum_{i=1}^{t-1} g_i$.

\begin{algorithm}[t]
\caption{Algorithm for OLO over Hilbert space $\H$ based on KT potential
\label{algorithm:kt-hilbert-space-olo}}
\begin{algorithmic}[1]
{
\REQUIRE{Initial endowment $\epsilon > 0$}
\FOR{$t=1,2,\dots$}
\STATE{Predict with $w_t \leftarrow \tfrac{1}{t} \left(\epsilon + \sum_{i=1}^{t-1} \langle g_i, w_i \rangle \right) \sum_{i=1}^{t-1} g_i$}
\STATE{Receive reward vector $g_t \in \H$ such that $\norm{g_t} \le 1$}
\ENDFOR
}
\end{algorithmic}
\end{algorithm}

The resulting algorithm is stated as
Algorithm~\ref{algorithm:kt-hilbert-space-olo}.  We derive a regret bound for
it as a very simple corollary of
Theorem~\ref{theorem:hilbert-space-olo-regret-bound} to the KT potential
\eqref{equation:kt-estimator-potential}. The only technical part of the proof,
in Appendix~\ref{section:corollaries_reductions}, is an upper bound on $F_t^*$
since it cannot be expressed as an elementary function.
\begin{corollary}[Regret Bound for Algorithm~\ref{algorithm:kt-hilbert-space-olo}]
\label{corollary:kt-hilbert-space-olo-regret} Let $\epsilon > 0$. Let
$\{g_t\}_{t=1}^\infty$ be any sequence of reward vectors in a Hilbert space
$\H$ such that $\norm{g_t} \le 1$.
Then Algorithm~\ref{algorithm:kt-hilbert-space-olo} satisfies
\[
\forall \, T \ge 0 \quad
\forall u \in \H \qquad \qquad
\Regret_T(u) \le \norm{u} \sqrt{T \ln\left(1 + \tfrac{24 T^2 \norm{u}^2}{\epsilon^2} \right)} + \epsilon \left(1 - \tfrac{1}{e\sqrt{ \pi T}} \right) \;.
\]
\end{corollary}
It is worth noting the elegance and extreme simplicity of
Algorithm~\ref{algorithm:kt-hilbert-space-olo} and contrast it with the
algorithms in \cite{Streeter-McMahan-2012, McMahan-Orabona-2014, Orabona-2013,
Orabona-2014}.  Also, the regret bound is
optimal~\cite{Streeter-McMahan-2012,Orabona-2013}.  The parameter $\epsilon$
can be safely set to any constant, e.g. $1$. Its role is equivalent to the
initial guess used in doubling tricks~\cite{Shalev-Shwartz-2011}.

\subsection{Learning with Expert Advice}
\label{section:kt-lea}

We will now construct an algorithm for \ac{LEA} based on the $\delta$-shifted
KT potential. We set $\delta$ to $T/2$, requiring the algorithm to know the
number of rounds $T$ in advance; we will fix this later with the standard
doubling trick.

\begin{algorithm}[t]
\begin{algorithmic}[1]
\caption{Algorithm for Learning with Expert Advice based on $\delta$-shifted KT potential
\label{algorithm:kt-experts}}
{
\REQUIRE{Number of experts $N$, prior distribution $\pi \in \Delta_N$, number of rounds $T$}
\FOR{$t=1,2,\dots,T$}
\STATE{For each $i \in [N]$, set $w_{t,i} \leftarrow \tfrac{\sum_{j=1}^{t-1} \widetilde g_{j,i}}{t+T/2} \left(1 + \sum_{j=1}^{t-1} \widetilde g_{j,i} w_{j,i} \right)$}
\STATE{For each $i \in [N]$, set $\widehat{p}_{t,i} \leftarrow \pi_i [w_{t,i}]_+$}
\STATE{Predict with $p_t \leftarrow
\begin{cases}
\widehat{p}_t/\norm{\widehat{p_t}}_1 & \text{if $\norm{\widehat p_t}_1 > 0$} \\
\pi & \text{if $\norm{\widehat p_t}_1 = 0$}
\end{cases}$}
\STATE{Receive reward vector $g_t \in [0,1]^N$}
\STATE{For each $i \in [N]$, set $\widetilde g_{t,i} \leftarrow \begin{cases}
g_{t,i} - \langle g_t, p_t \rangle & \text{if $w_{t,i} > 0$} \\
[g_{t,i} - \langle g_t, p_t \rangle]_+ & \text{if $w_{t,i} \le 0$}
\end{cases}$}
\ENDFOR
}
\end{algorithmic}
\end{algorithm}

To use the construction in Section~\ref{section:reduction-experts}, we need an
OLO algorithm for the 1-d Hilbert space $\R$.  Using the $\delta$-shifted KT
potentials, the algorithm predicts for any sequence $\{\widetilde
g_t\}_{t=1}^\infty$ of reward
\[
w_t
= \beta_t \Wealth_{t-1}
= \beta_t \left(1 + \sum_{j=1}^{t-1} \widetilde g_j w_j \right)
= \frac{\sum_{i=1}^{t-1} \widetilde g_i}{T/2+t} \left(1 + \sum_{j=1}^{t-1} \widetilde g_j w_j \right) \; .
\]
Then, following the construction in Section~\ref{section:reduction-experts}, we
arrive at the final algorithm, Algorithm~\ref{algorithm:kt-experts}.
We can derive a regret bound for Algorithm~\ref{algorithm:kt-experts} by
applying Theorem~\ref{theorem:regret-bound-experts} to the $\delta$-shifted KT
potential.
\begin{corollary}[Regret Bound for
Algorithm~\ref{algorithm:kt-experts}] \label{corollary:kt-experts-regret} Let
$N \ge 2$ and $T \ge 0$ be integers. Let $\pi \in \Delta_N$ be a prior.
Then Algorithm~\ref{algorithm:kt-experts} with input $N,\pi,T$
for any rewards vectors $g_1, g_2, \dots, g_T \in [0,1]^N$ satisfies
\[
\forall u \in \Delta_N \qquad \qquad \Regret_T(u) \le \sqrt{3T (3 + \KL{u}{\pi})} \; .
\]
\end{corollary}
Hence, the Algorithm~\ref{algorithm:kt-experts} has \emph{both} the best known
guarantee on worst-case regret and per-round time complexity, see
Table~\ref{table:bounds}. Also, it has the advantage of being very simple.

The proof of the corollary is in the
Appendix~\ref{section:corollaries_reductions}.  The only technical part of the
proof is an upper bound on $f_t^{-1}(x)$, which we conveniently do by lower
bounding $F_t(x)$.

The reason for using the shifted potential comes from the analysis of
$f_t^{-1}(x)$. The unshifted algorithm would have a $O(\sqrt{T (\log T +
\KL{u}{\pi}})$ regret bound; the shifting improves the bound to $O(\sqrt{T (1 +
\KL{u}{\pi}})$.  By changing $T/2$ in Algorithm~\ref{algorithm:kt-experts} to
another constant fraction of $T$, it is possible to trade-off between the two
constants $3$ present in the square root in the regret upper bound.

The requirement of knowing the number of rounds $T$ in advance can be lifted by
the standard doubling trick~\cite[Section 2.3.1]{Shalev-Shwartz-2011},
obtaining an anytime guarantee with a bigger leading constant,
\[
\forall \, T \ge 0 \quad \forall u \in \Delta_N \qquad \qquad
\Regret_T(u) \le \tfrac{\sqrt{2}}{\sqrt{2} - 1} \sqrt{3T (3 + \KL{u}{\pi})} \; .
\]

\section{Discussion of the Results}
\label{section:discussion}

\begin{figure}[t]
\centering
\subfigure{\includegraphics[width=0.32\textwidth]{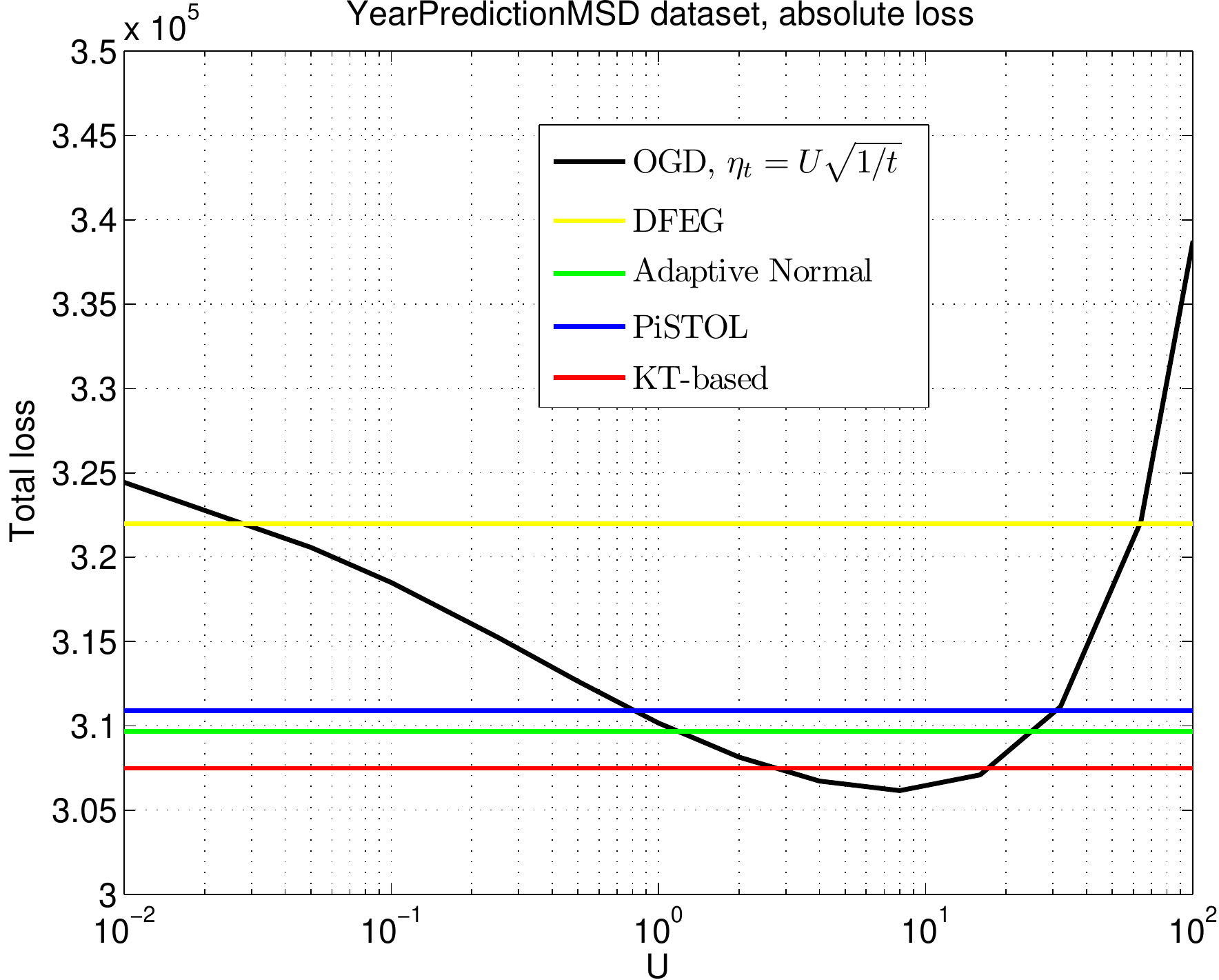}}
\subfigure{\includegraphics[width=0.32\textwidth]{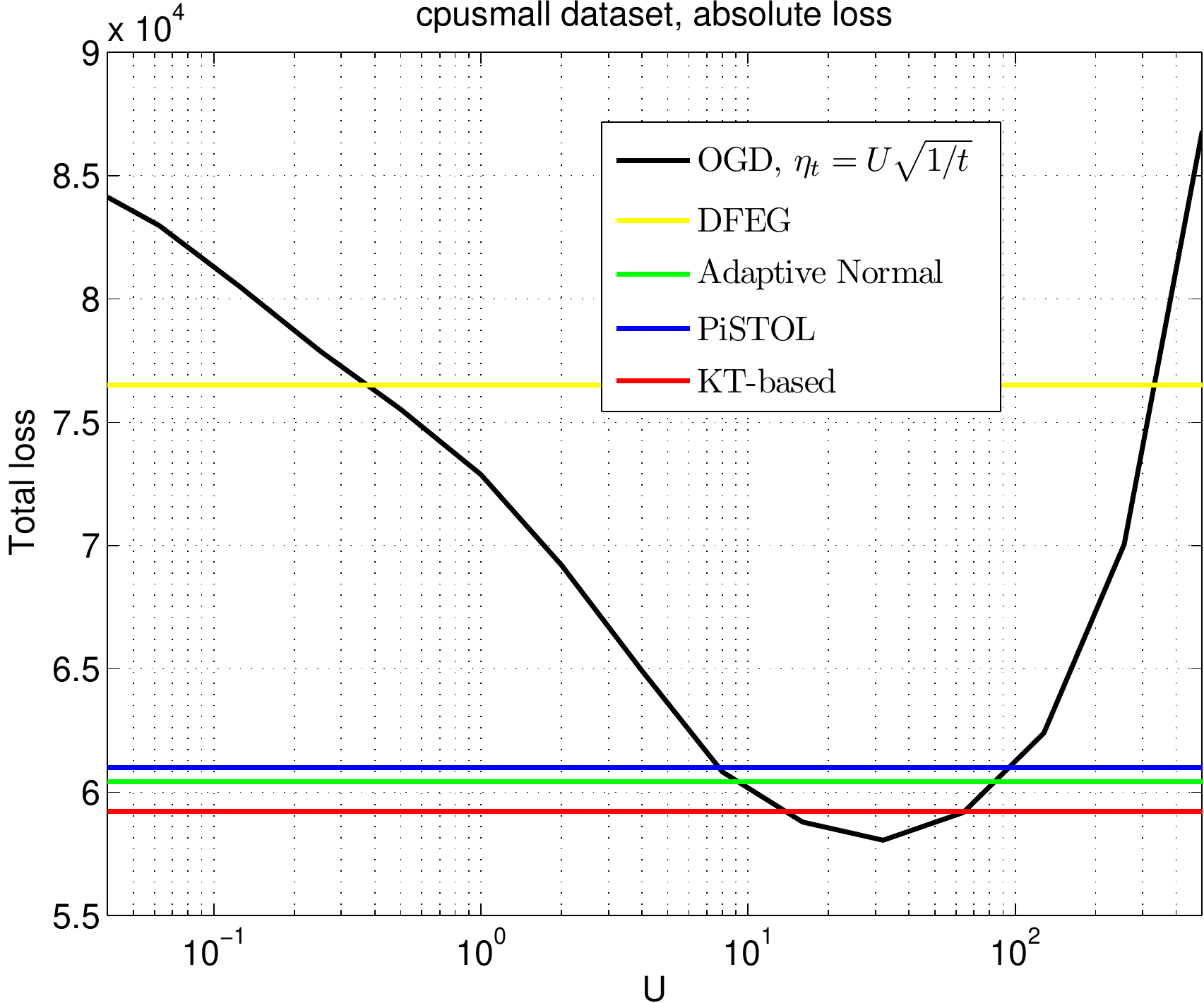}}
\subfigure{\includegraphics[width=0.32\textwidth]{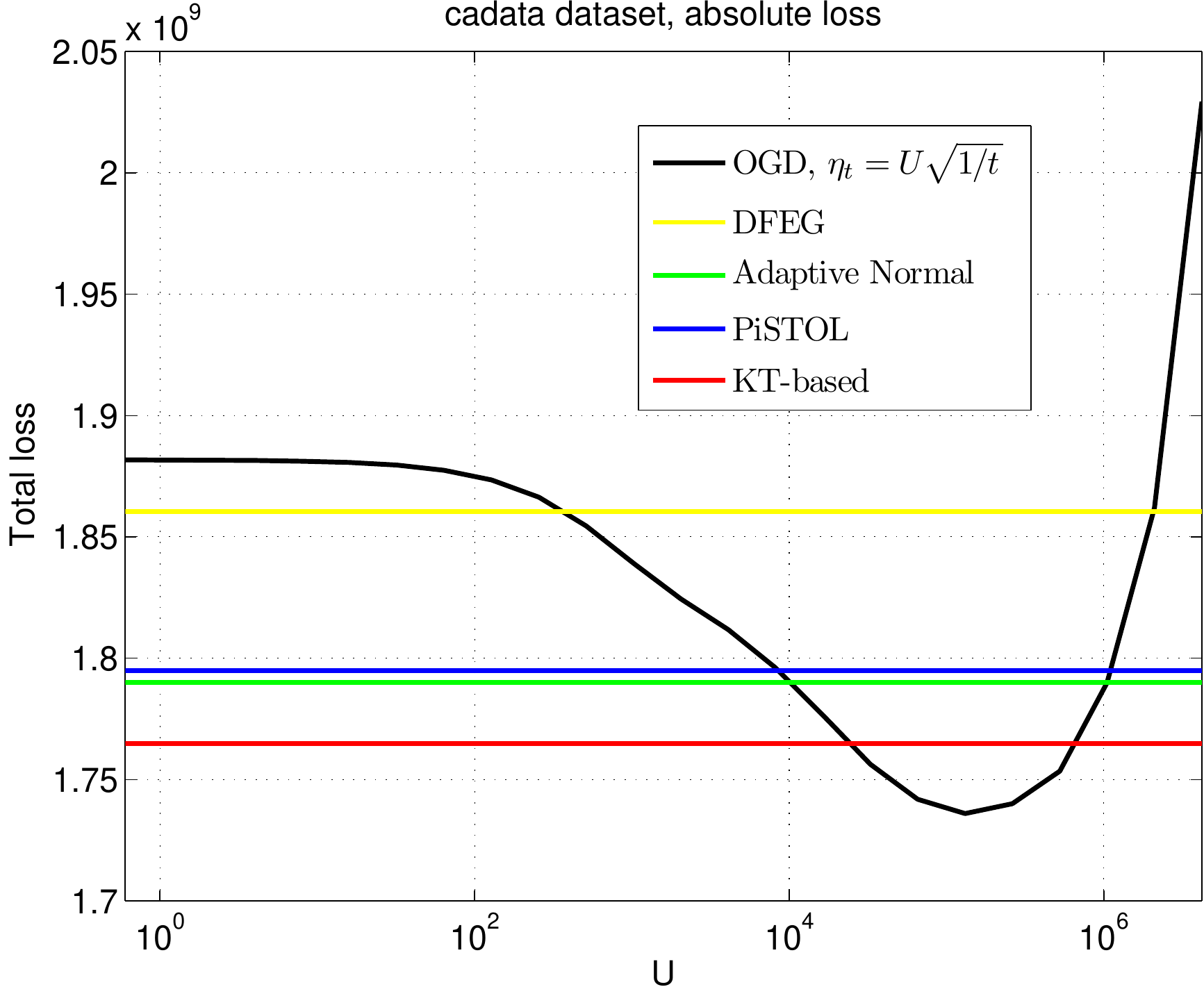}}
\caption{\footnotesize{Total loss versus learning rate parameter of \ac{OGD}
(in log scale), compared with parameter-free algorithms
DFEG~\cite{Orabona-2013}, Adaptive Normal~\cite{McMahan-Orabona-2014},
PiSTOL~\cite{Orabona-2014} and the KT-based
Algorithm~\ref{algorithm:kt-hilbert-space-olo}.}}
\label{fig:exp_olo}
\end{figure}

\begin{figure}[t]
\centering
\subfigure{\includegraphics[width=0.32\textwidth]{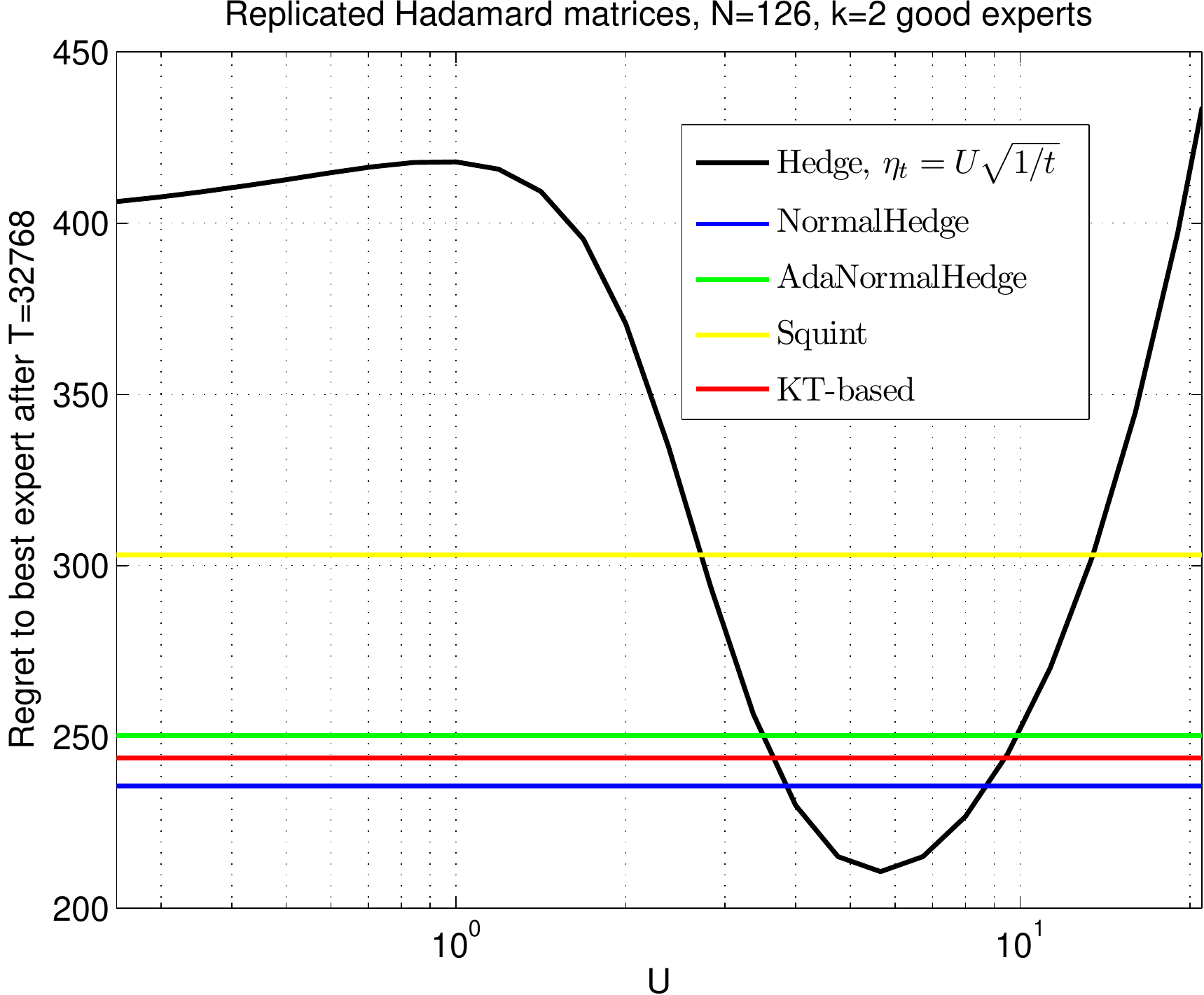}}
\subfigure{\includegraphics[width=0.32\textwidth]{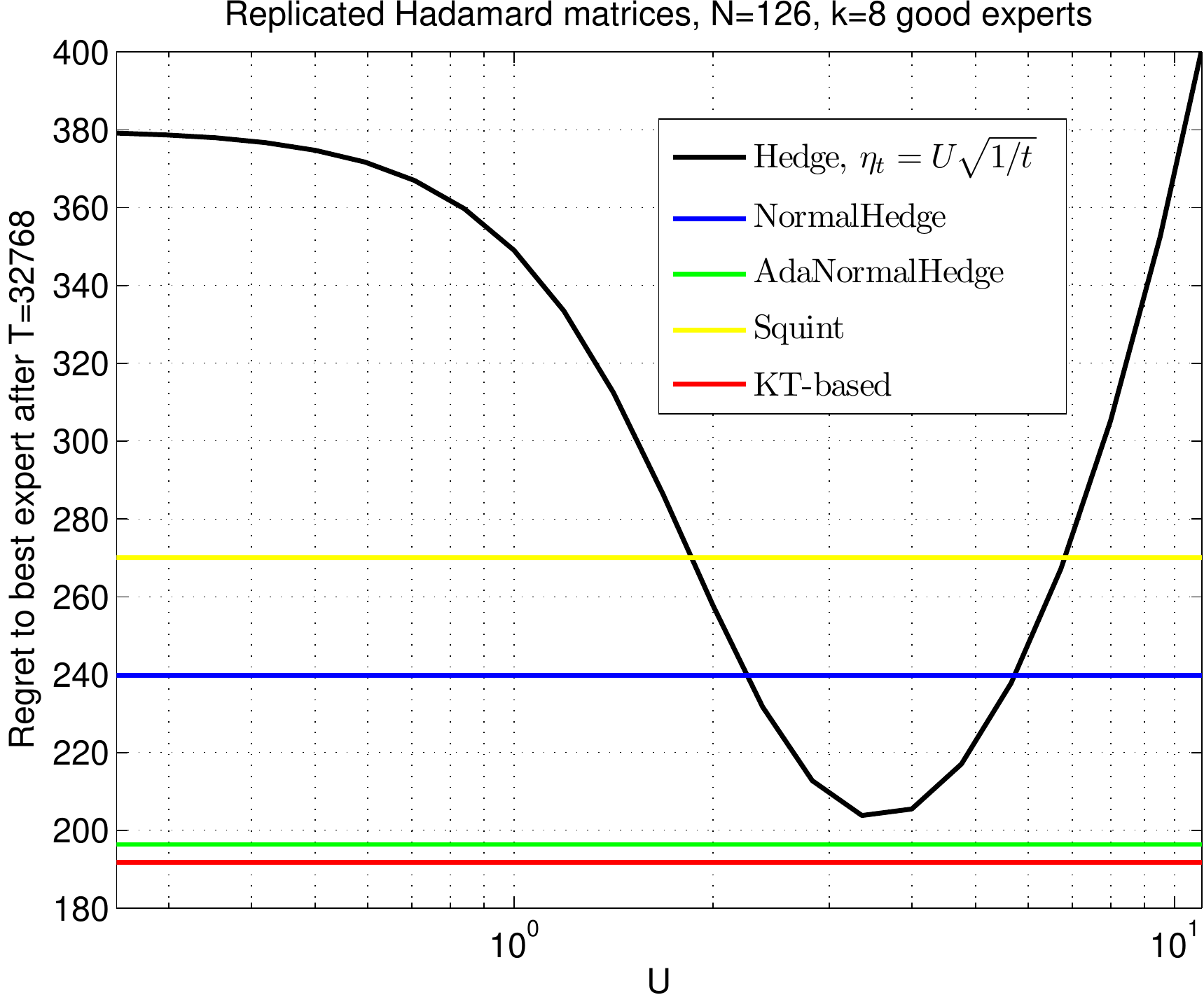}}
\subfigure{\includegraphics[width=0.32\textwidth]{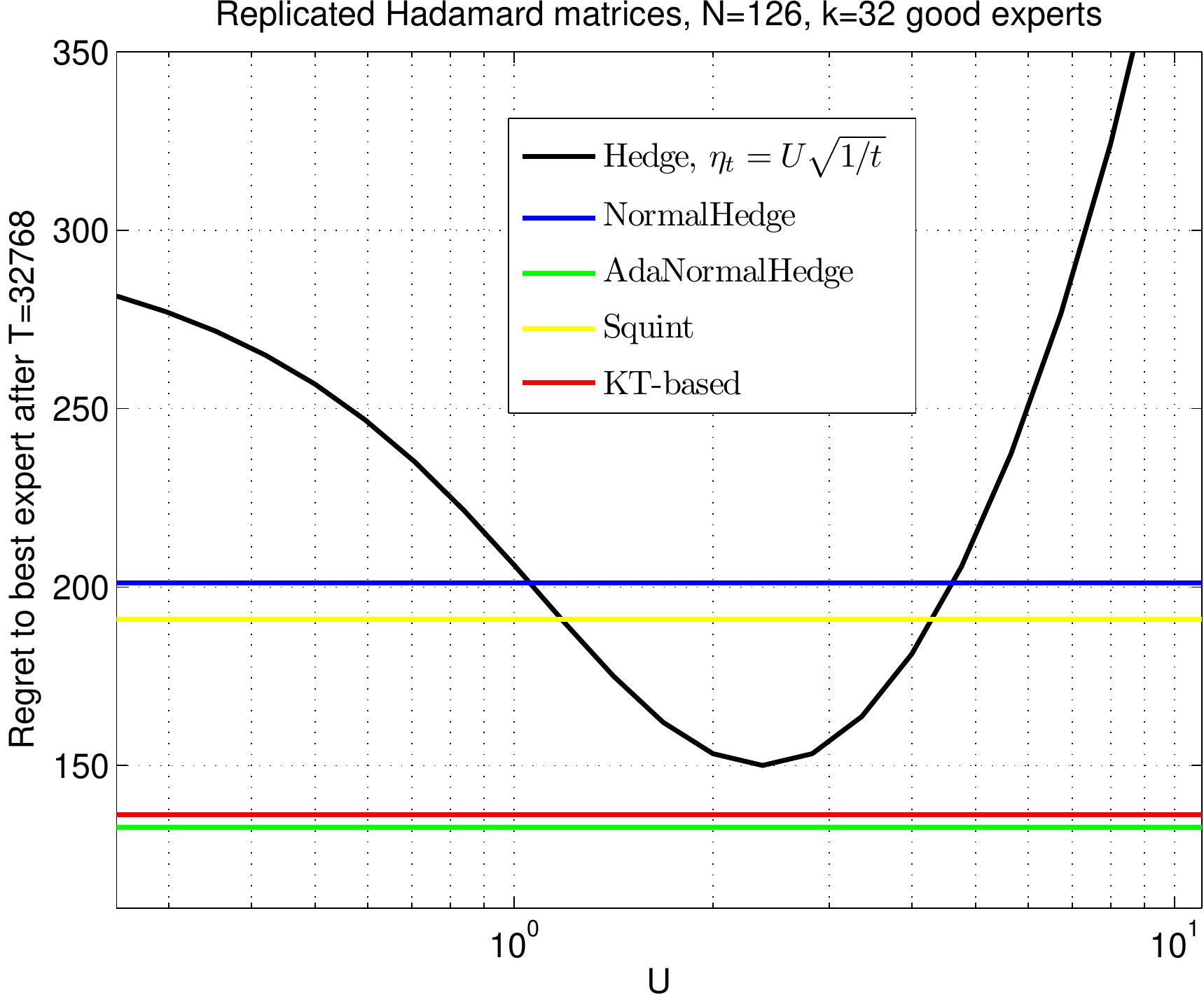}}
\caption{\footnotesize{Regrets to the best expert after $T = 32768$ rounds,
versus learning rate parameter of Hedge (in log scale). The ``good'' experts
are $\epsilon=0.025$ better than the others. The competitor algorithms are
NormalHedge~\cite{Chaudhuri-Freund-Hsu-2009},
AdaNormalHedge~\cite{Luo-Schapire-2015}, Squint~\cite{Koolen-van-Erven-2015},
and the KT-based Algorithm~\ref{algorithm:kt-experts}. $\pi_i=1/N$ for all
algorithms.}}
\label{fig:exp_lea}
\end{figure}

We have presented a new interpretation of parameter-free algorithms as
coin betting algorithms. This interpretation, far from being just a
mathematical gimmick, reveals the \emph{common} hidden structure of previous
parameter-free algorithms for both OLO and LEA and also allows the design of new algorithms. For
example, we show that the characteristic of parameter-freeness is just a
consequence of having an algorithm that guarantees the maximum reward possible.
The reductions in Sections~\ref{section:reduction_hilbert}
and~\ref{section:reduction-experts} are also novel and they are in a certain
sense optimal. In fact, the obtained
Algorithms~\ref{algorithm:kt-hilbert-space-olo} and~\ref{algorithm:kt-experts}
achieve the optimal worst case upper bounds on the regret, see
\cite{Streeter-McMahan-2012,Orabona-2013} and \cite{Cesa-Bianchi-Lugosi-2006}
respectively.

We have also run an empirical evaluation to show that the theoretical
difference between classic online learning algorithms and parameter-free ones
is real and not just theoretical. In Figure~\ref{fig:exp_olo}, we have used
three regression datasets\footnote{Datasets available at
\url{https://www.csie.ntu.edu.tw/~cjlin/libsvmtools/datasets/}.}, and solved
the \ac{OCO} problem through \ac{OLO}. In all the three cases, we have used the
absolute loss and normalized the input vectors to have L2 norm equal to 1. From
the empirical results, it is clear that the optimal learning rate is completely
data-dependent, yet \emph{parameter-free algorithms have performance very close
to the unknown optimal tuning of the learning rate}. Moreover, the KT-based
Algorithm~\ref{algorithm:kt-hilbert-space-olo} seems to dominate all the other
similar algorithms.

For \ac{LEA}, we have used the synthetic setting
in~\cite{Chaudhuri-Freund-Hsu-2009}. The dataset is composed of Hadamard
matrices of size 64, where the row with constant values is removed, the rows
are duplicated to 126 inverting their signs, $0.025$ is subtracted to $k$ rows,
and the matrix is replicated in order to generate $T=32768$ samples. For more
details, see~\cite{Chaudhuri-Freund-Hsu-2009}. Here, the KT-based algorithm is
the one in Algorithm~\ref{algorithm:kt-experts}, where the term $T/2$ is
removed, so that the final regret bound has an additional $\ln T$ term.  Again,
we see that the parameter-free algorithms have a performance close or
\emph{even better} than Hedge with an oracle tuning of the learning rate, with
no clear winners among the parameter-free algorithms.

Notice that since the adaptive Kelly strategy based on KT estimator is very
close to optimal, the only possible improvement is to have a data-dependent
bound, for example like the ones in~\cite{Orabona-2014, Koolen-van-Erven-2015,
Luo-Schapire-2015}.  In future work, we will extend our definitions and
reductions to the data-dependent case.
\textbf{Acknowledgments.} The authors thank Jacob Abernethy, Nicol\`{o}
Cesa-Bianchi, Satyen Kale, Chansoo Lee, Giuseppe Molteni, and Manfred Warmuth for useful
discussions on this work.

\bibliographystyle{plainnat}
\bibliography{learning}

\appendix
\section{From Log Loss to Wealth}
\label{section:logloss-to-wealth}

Guarantees for betting or sequential investement algorithm are often expressed
as upper bounds on the regret with respect to the log loss.  Here, for the sake
of completeness, we show how to convert such a guarantee to a lower bound on
the wealth of the corresponding betting algorithm.

We consider the problem of predicting a binary outcome.  The algorithm predicts
at each round probability $p_t \in [0,1]$. The adversary generates a sequences
of outcomes $x_t \in \{0,1\}$ and the algorithm's loss is
\[
\ell(p_t,x_t) = -x_t \ln p_t -(1-x_t) \ln (1-p_t) \; .
\]
We define the regret with respect to a fixed probability vector $\beta$ as
\[
\Regret^{\mathrm{logloss}}_T = \sum_{t=1}^T \ell(p_t,x_t) - \min_{\beta \in [0,1]} \sum_{t=1}^T \ell(\beta,x_t) \; .
\]

\begin{lemma}
Assume that an algorithm that predicts $p_t$ guarantees
$\Regret^{\mathrm{logloss}}_T \leq R_T$.  Then, the coin betting strategy with
endowement $\epsilon$ and $\beta_t = 2 p_{t}-1$ guarantees
\[
\Wealth_T \ge \epsilon \exp\left(T \cdot \KL{\frac{1}{2} + \frac{\sum_{t=1}^T g_t}{2 T}}{\frac{1}{2}} - R_T \right)
\]
against any sequence of outcomes $g_t \in [-1,+1]$.
\end{lemma}

\begin{proof}
Define $x_t=\tfrac{1+g_t}{2}$. We have
\begin{align*}
\ln \Wealth_T
& = \ln (\Wealth_{t-1} + w_t g_t) \\
& = \ln (\Wealth_{t-1}(1 + g_t \beta_t))\\
& = \ln \epsilon \prod_{t=1}^T (1 + g_t\beta_t) \\
& = \ln \epsilon + \sum_{t=1}^T \ln (1 + g_t\beta_t)\\
& \ge \ln \epsilon +  \sum_{t=1}^T \left( \frac{1+g_t}{2} \right) \ln \left(1 + \beta_t\right) + \left( \frac{1-g_t}{2} \right) \ln \left(1 - \beta_t \right) \\
& =  \ln \epsilon + \sum_{t=1}^T \left( \frac{1+g_t}{2} \right) \ln \left(2p_t \right) + \left( \frac{1-g_t}{2} \right) \ln \left(2 (1 - p_t) \right) \\
& =  \ln \epsilon + T \ln(2) + \sum_{t=1}^T \left( \frac{1+g_t}{2} \right) \ln (p_t) + \left( \frac{1-g_t}{2} \right) \ln (1 - p_t) \\
& =  \ln \epsilon + T \ln(2) - \sum_{t=1}^T \ell(p_t, x_t) \\
& =  \ln \epsilon + T \ln(2) - \Regret^{\mathrm{logloss}}_T - \min_{\beta \in [0,1]} \sum_{t=1}^T \ell(\beta,x_t) \\
& \ge  \ln \epsilon + T \ln(2) - R_T - \min_{\beta \in [0,1]} \sum_{t=1}^T \ell(\beta,x_t) \; ,
\end{align*}
where the first inequality is due to the concavity of $\ln$ and the second one
is due to the assumption of the regret.

It is easy to see that the $\beta^*=\argmin_{\beta \in [0,1]} \sum_{t=1}^T
\ell(\beta,x_t)=\tfrac{\sum_{t=1}^T x_t}{T}$. Hence, we have
\[
\min_{\beta \in [0,1]} \sum_{t=1}^T \ell(\beta,x_t) = T \left( - \beta^* \ln \beta^* - (1-\beta^*) \ln (1-\beta^*)\right) \; .
\]
Also, we have that for any $\beta \in [0,1]$
\[
- \beta \ln \beta - (1-\beta) \ln (1-\beta) = - \KL{\beta}{\frac{1}{2}} + \ln 2 \; .
\]

Putting all together, we have the stated lemma.
\end{proof}

The lower bound on the wealth of the adaptive Kelly betting based on the KT
estimator is obtained simply by the stated Lemma and reminding that the log
loss regret of the KT estimator is upper bounded by $\frac{1}{2}\ln T + \ln 2$.

\section{Optimal Betting Fraction}
\label{section:optimal-betting-fraction}

\begin{theorem}[Optimal Betting Fraction]
\label{theorem:optimal-betting-fraction}
Let $x \in \R$. Let $F:[x-1,x+1] \to \R$ be a logarithmically convex function. Then,
\[
\argmin_{\beta \in (-1,1)} \max_{g \in [-1,1]} \ \frac{F(x+g)}{1 + \beta g}
= \frac{F(x+1) - F(x-1)}{F(x+1) + F(x-1)} \; .
\]
Moreover, $\beta^*=\frac{F(x+1) - F(x-1)}{F(x+1) + F(x-1)}$ satisfies
\[
\ln(F(x+1)) - \ln(1 + \beta^*) =  \ln(F(x-1)) - \ln(1 - \beta^*) \; .
\]
\end{theorem}

\begin{proof}
We define the functions $h,f:[-1,1] \times (-1,1) \to \R$ as
\begin{align*}
h(g, \beta) & = \frac{F(x+g)}{1 + \beta g} &
& \text{and} &
f(g, \beta) & = \ln (h(g,\beta)) = \ln(F(x+g)) - \ln(1 + \beta g) \; .
\end{align*}
Clearly, $\argmin_{\beta \in (-1,1)} \max_{g \in [-1,1]} h(g,\beta) =
\argmin_{\beta \in (-1,1)} \max_{g \in [-1,1]} f(g,\beta)$ and we can work with
$f$ instead of $h$. The function $h$ is logarithmically convex in $g$ and thus
$f$ is convex in $g$. Therefore,
\[
\forall \beta \in (-1,1) \qquad \qquad
\max_{g \in [-1,1]} f(g,\beta) = \max \left\{ f(+1,\beta), f(-1,\beta) \right\} \; .
\]
Let $\phi(\beta) = \max \left\{ f(+1,\beta), f(-1,\beta) \right\}$. We seek to
find the $\argmin_{\beta \in (-1,1)} \phi(\beta)$. Since $f(+1,\beta)$ is
decreasing in $\beta$ and $f(-1,\beta)$ is increasing in $\beta$, the minimum
of $\phi(\beta)$ is at a point $\beta^*$ such that $f(+1,\beta^*) =
f(-1,\beta^*)$.  In other words, $\beta^*$ satisfies
\[
\ln(F(x+1)) - \ln(1 + \beta^*) =  \ln(F(x-1)) - \ln(1 - \beta^*) \; .
\]
The only solution of this equation is
\[
\beta^* = \frac{F(x+1) - F(x-1)}{F(x+1) + F(x-1)} \; .
\]
\end{proof}

\begin{theorem}
\label{thm:exp_x2}
The functions $F_t(x)=\epsilon\exp(\tfrac{x^2}{2t}- \frac{1}{2}\sum_{i=1}^t \tfrac{1}{i})$ are excellent coin betting potentials.
\end{theorem}
\begin{proof}
The first and second properties of Definition~\ref{definition:potential} are
trivially true.  For the third property, we first use
Theorem~\ref{theorem:optimal-betting-fraction} to have
\[
\ln(1+\beta_t g) - \ln F_t(x+g) \geq \ln(1+\beta_t) - \ln F_t(x+1) = \ln \frac{2}{F_t(x+1)+F_t(x-1)},
\]
where the definition of $\beta_t$ is from \eqref{equation:potential-based-strategy}.
Hence, we have
\begin{align*}
\ln(1+\beta_t g) - \ln F_t(x+g)  +\ln F_{t-1}(x)
&\geq \ln \frac{2}{F_t(x+1)+F_t(x-1)} +\ln F_{t-1}(x) \\
&= -\frac{x^2+1}{2t} +\frac{1}{2}\sum_{i=1}^t \frac{1}{i} -\ln \cosh \frac{x}{t} + \frac{x^2}{2(t-1)} - \frac{1}{2}\sum_{i=1}^{t-1} \frac{1}{i}\\
&= -\frac{x^2}{2t} -\ln \cosh \frac{x}{t} + \frac{x^2}{2(t-1)}\\
&\geq -\frac{x^2}{2t}  -\frac{x^2}{2 t^2} + \frac{x^2}{2(t-1)} \\
&\geq -\frac{x^2}{2t}  -\frac{x^2}{2 t (t-1)} + \frac{x^2}{2(t-1)} =0,
\end{align*}
where in the second inequality we have used the elementary inequality $\ln
\cosh x \leq \tfrac{x^2}{2}$.

The fourth property of Definition~\ref{definition:potential} is also true
because $F_t(x)$ is of the form $h(x^2)$ with $h(\cdot)$
convex~\cite{McMahan-Orabona-2014}.
\end{proof}

\section{Proof of Lemma~\ref{lemma:recursion_hilbert}}
\label{section:hilbert-space-reduction}

First we state the following Lemma from~\cite{McMahan-Orabona-2014} and
reported here with our notation for completeness.

\begin{lemma}[Extremes]
\label{lemma:extremes}
Let $h:(-a,a) \to \R$ be an even twice-differentiable function that
satisfies $x \cdot h''(x) \ge h'(x)$ for all $x \in [0,a)$. Let $c:[0,\infty) \times [0, \infty) \to \R$
be an arbitrary function. Then, if vectors $u,v \in \H$ satisfy $\|u\| + \|v\| < a$, then
\begin{multline}
\label{equation:lemma-extremes-1}
c(\norm{u}, \norm{v}) \cdot \langle u, v \rangle - h(\norm{u+v})
\ge \min \left\{ c(\norm{u}, \norm{v}) \cdot \norm{u} \cdot \norm{v} - h(\norm{v} + \norm{v}), \right. \\
\left. - c(\norm{u}, \norm{v}) \cdot \norm{u} \cdot \norm{v} - h(\norm{u} - \norm{v}) \right\} \; .
\end{multline}
\end{lemma}
\begin{proof}
If $u$ or $v$ is zero, the inequality \eqref{equation:lemma-extremes-1} clearly
holds. From now on we assume that $u,v$ are non-zero. Let $\alpha$ be the cosine
of the angle of between $u$ and $v$. More formally,
\[
\alpha = \frac{\langle u, v \rangle}{\norm{u} \cdot \norm{v}} \; .
\]
With this notation, the left-hand side of \eqref{equation:lemma-extremes-1} is
\[
f(\alpha) = c(\norm{u}, \norm{v}) \cdot \alpha \norm{u} \cdot \norm{v} - h(\sqrt{\norm{u}^2 + \norm{v}^2 + 2 \alpha \norm{u} \cdot \norm{v}}) \; .
\]
Since $h$ is even, the inequality \eqref{equation:lemma-extremes-1} is equivalent to
\[
\forall \alpha \in [-1,1] \qquad \qquad f(\alpha) \ge \min \left\{f(+1), f(-1)\right\} \; .
\]
The last inequality is clearly true if $f:[-1,1] \to \R$ is concave. We now
check that $f$ is indeed concave, which we prove by showing that the second
derivative is non-positive. The first derivative of $f$ is
\[
f'(\alpha) = c(\norm{u}, \norm{v}) \cdot \|u\| \cdot \|v\| - \frac{h'(\sqrt{\|u\|^2 + \|v\|^2 + 2 \alpha \|u\| \cdot \|v\|}) \cdot \|u\| \cdot \|v\|}{\sqrt{\|u\|^2 + \|v\|^2 + 2 \alpha \|u\| \cdot \|v\|}} \; .
\]
The second derivative of $f$ is
\begin{multline*}
f''(\alpha) = - \frac{\|u\|^2 \cdot \|v\|^2}{\|u\|^2 + \|v\|^2 + 2 \alpha \|u\| \cdot \|v\|} \\
 \cdot \left( h''(\sqrt{\|u\|^2 + \|v\|^2 + 2 \alpha \|u\| \cdot \|v\|})  - \frac{h'(\sqrt{\|u\|^2 + \|v\|^2 + 2 \alpha \|u\| \cdot \|v\|})}{\sqrt{\|u\|^2 + \|v\|^2 + 2\alpha \|u\| \cdot \|v\|}}  \right) \; .
\end{multline*}
If we consider $x = \sqrt{\|u\|^2 + \|v\|^2 + 2 \alpha \|u\| \cdot \|v\|}$, the
assumption $x \cdot h''(x) \ge h'(x)$ implies that $f''(\alpha)$ is non-positive.
This finishes the proof of the inequality \eqref{equation:lemma-extremes-1}.
\end{proof}

We also need the following technical Lemma whose proof relies mainly on
property (d) of Definition~\ref{definition:potential}.
\begin{lemma}
\label{lemma:recursion_hilbert}
Let $\{F_t\}_{t=0}^\infty$ be a sequence of excellent coin betting potentials.
Let $g_1, g_2, \dots, g_t$ be vectors in a Hilbert space $\H$ such that
$\norm{g_1}, \norm{g_2}, \dots, \norm{g_t} \le 1$. Let $\beta_t$
be defined by \eqref{equation:potential-based-strategy-hilbert-space}
and let $x = \sum_{i=1}^{t-1} g_i$. Then,
\[
\left(1 + \beta_t \frac{\langle g_t, x \rangle}{\norm{x}} \right) F_{t-1}(\norm{x})
\ge F_t(\norm{x + g_t}) \; .
\]
\end{lemma}
\begin{proof}
Since $F_t(x)$ is an excellent coin betting potential, it satisfies $x
F_t''(x) \ge F_t'(x)$. Hence,
\begin{align*}
&\left(1 + \beta_t \frac{\langle g_t, x \rangle}{\norm{x}} \right) F_{t-1}(\norm{x}) - F_t(\norm{x + g_t}) \\
&\quad = F_{t-1}(\norm{x}) + \beta_t \frac{\langle g_t, x \rangle}{\norm{x}} F_{t-1}(\norm{x}) - F_t(\norm{x + g_t}) \\
&\quad \ge F_{t-1}(\norm{x})+\min_{r \in \{-1,1\}} \beta_t r \norm{g_t} F_{t-1}(\norm{x}) - F_t(\norm{x} + r \norm{g_t}) \\
&\quad =\min_{r \in \{-1,1\}} \left(1 + \beta_t r \norm{g_t}\right) F_{t-1}(\norm{x}) - F_t(\norm{x} + r \norm{g_t}) \\
&\quad \ge 0 \; .
\end{align*}
If $x \neq 0$, the first inequality comes from Lemma~\ref{lemma:extremes} with
$c(z,\cdot) = \frac{F_{t-1}(z+1) - F_{t-1}(z-1)}{F_{t-1}(z+1) + F_{t-1}(z-1)} F_{t-1}(z) / z$ and
$h(z) = F_t(z)$, $u=g_t$, $v=x$.
If $x=0$ then, according to
\eqref{equation:potential-based-strategy-hilbert-space}, $\beta_t = 0$ and the
first inequality trivially holds. The second inequality follows from the
property (c) of a coin betting potential.
\end{proof}

\begin{proof}[Proof of Theorem~\ref{theorem:hilbert-space-olo-regret-bound}]
First, by induction on $t$ we show that
\begin{equation}
\label{equation:wealth-lower-bound-hilbert-space}
\Wealth_t \ge F_t\left(\norm{\sum_{t=1}^T g_t} \right) \; .
\end{equation}
The base case $t=0$ is trivial, since both sides of the inequality are equal to
$\epsilon$.  For $t \ge 1$, if we let $x = \sum_{i=1}^{t-1} g_i$, we have
\begin{align*}
\Wealth_t
&= \langle g_t, w_t \rangle + \Wealth_{t-1}
= \left(1 + \beta_t \frac{\langle g_t, x \rangle}{\norm{x}} \right) \Wealth_{t-1} \\
&\ge \left(1 + \beta_t \frac{\langle g_t, x \rangle}{\norm{x}} \right) F_{t-1}(\norm{x})
\stackrel{\text{\textbf{(*)}}}{\ge} F_t(\norm{x + g_t})
= F_t\left(\norm{\sum_{i=1}^t g_i} \right) \; .
\end{align*}
The inequality marked with $(*)$ follows from
Lemma~\ref{lemma:recursion_hilbert}.

This establishes \eqref{equation:wealth-lower-bound-hilbert-space},
from which we immediately have a reward lower bound
\begin{equation}
\label{equation:hilbert-space-olo-reward-lower-bound}
\Reward_T
= \sum_{t=1}^T \langle g_t, w_t \rangle
= \Wealth_T \ - \ \epsilon
\ge F_T\left(\norm{\sum_{t=1}^T g_t} \right) \ - \ \epsilon \; .
\end{equation}
We apply Lemma~\ref{lemma:reward-regret} to the function $F(x) = F_T(\norm{x}) -
\epsilon$ and we are almost done. The only remaining property we need is that if
$F$ is an even function then the Fenchel conjugate of $F(\norm{\cdot})$ is
$F^*(\norm{\cdot})$; see \citet[Example 13.7]{Bauschke-Combettes-2011}.
\end{proof}

\section{Proof of Theorem~\ref{theorem:regret-bound-experts}}
\label{section:appendix-expert-reduction}

\begin{proof}
We first prove that $\sum_{i=1}^N \pi_i \widetilde g_{t,i} w_{t,i} \le 0$. Indeed,
\begin{align*}
\sum_{i=1}^N \pi_i \widetilde g_{t,i} w_{t,i}
& = \sum_{i \, : \, \pi_i w_{t,i} > 0} \pi_i [w_{t,i}]_+ (g_{t,i} - \langle g_t, p_t \rangle)  \ + \ \sum_{i \, : \, \pi_i w_{t,i} \le 0} \pi_i w_{t,i} [g_{t,i} - \langle g_t, p_t \rangle ]_+ \\
& = \norm{\widehat p_t}_1 \sum_{i=1}^N p_{t,i} (g_{t,i} - \langle g_t, p_t \rangle)  \ + \ \sum_{i \, : \, \pi_i w_{t,i} \le 0} \pi_i w_{t,i} [g_{t,i} - \langle g_t, p_t\rangle]_+ \\
& = 0 \ + \ \sum_{i \, : \, \pi_i w_{t,i} \le 0} \pi_i w_{t,i} [g_{t,i} - \langle g_t, p_t\rangle]_+
\ \le 0 \; .
\end{align*}
The first equality follows from definition of $g_{t,i}$. To see the second
equality, consider two cases: If $\pi_i w_{t,i} \le 0$ for all $i$ then
$\norm{\widehat p_t}_1 = 0$ and therefore both $\norm{\widehat p_t}_1
\sum_{i=1}^N p_{t,i} (g_{t,i} - \langle g_t, p_t \rangle)$ and $\sum_{i \, : \,
\pi_i w_{t,i} > 0} \pi_i [w_{t,i}]_+ (g_{t,i} - \langle g_t, p_t \rangle)$ are
trivially zero.  If $\norm{\widehat p_t}_1 > 0$ then $\pi_i [w_{t,i}]_+ =
\widehat p_{t,i} = \norm{\widehat p_t}_1 p_{t,i}$ for all $i$.

From the assumption on $A$, we have, for any sequence
$\{\widetilde g_t\}_{t=1}^\infty$ such that $\widetilde g_t \in [-1,1]$, satisfies
\begin{equation}
\label{equation:experts-one-dimensional-assumption}
\Wealth_t = 1 + \sum_{i=1}^t \widetilde g_i w_i \ge F_t\left(\sum_{i=1}^t \widetilde g_i\right) \; .
\end{equation}
Inequality $\sum_{i=1}^N \pi_i \widetilde g_{t,i} w_{t,i} \le 0$ and \eqref{equation:experts-one-dimensional-assumption} imply
\begin{equation}
\label{equation:bounded-potential}
\sum_{i=1}^N  \pi_i F_T \left(\sum_{t=1}^T \widetilde g_{t,i} \right)
\le 1 + \sum_{i=1}^N \pi_i \sum_{t=1}^T  \widetilde g_{t,i} w_{t,i} \le 1 \; .
\end{equation}
Now, let $\widetilde G_{T,i} =
\sum_{t=1}^T \widetilde g_{t,i}$. For any competitor $u \in \Delta_N$,
\begingroup
\allowdisplaybreaks
\begin{align*}
\allowdisplaybreaks
&\Regret_T(u)
= \sum_{t=1}^T \langle g_t, u - p_t \rangle
= \sum_{t=1}^T \sum_{i=1}^N u_i \left(g_{t,i} - \langle g_t, p_t \rangle \right) \\
& \le \sum_{t=1}^T \sum_{i=1}^N u_i \widetilde g_{t,i} \qquad \text{(by definition of $\widetilde g_{t,i}$)} \\
& \le \sum_{i=1}^N u_i \left|\widetilde G_{T,i}\right| \qquad \text{(since $u_i \ge 0, i=1,\ldots, N$)}  \\
& = \sum_{i=1}^N u_i f_T^{-1}\left(\ln [F_T(\widetilde G_{T,i})] \right)  \qquad \text{(since $F_T(x) = \exp(f_T(x))$ is even)} \\
& \le f_T^{-1}\left(\sum_{i=1}^N u_i \ln \left[ F_T(\widetilde G_{T,i}) \right]\right) \qquad \text{(by concavity of $f_T^{-1}$)} \\
& = f_T^{-1}\left(\sum_{i=1}^N u_i \left\{\ln \left[\frac{u_i}{\pi_i}\right] +\ln \left[ \frac{\pi_i}{u_i} F_T(\widetilde G_{T,i}) \right] \right\} \right)
= f_T^{-1}\left(\KL{u}{\pi}+\sum_{i=1}^N u_i\ln \left[\frac{\pi_i}{u_i} F_T(\widetilde G_{T,i}) \right]\right) \\
& \le f_T^{-1}\left(\KL{u}{\pi}+\ln \left(\sum_{i=1}^N \pi_i F_T(\widetilde G_{T,i}) \right)\right) \qquad \text{(by concavity of $\ln(\cdot)$)} \\
& \le f_T^{-1}\left(\KL{u}{\pi}\right) \qquad \text{(by \eqref{equation:bounded-potential})}. \qedhere
\end{align*}
\endgroup
\end{proof}

\section{Properties of Krichevsky-Trofimov Potential}
\label{section:properties-kt-potential}

\begin{lemma}[Analytic Properties of KT potential]
\label{lemma:kt-potential-analytic-properties}
Let $a > 0$. The function $F:(-a,a) \to \R_+$,
\[
F(x) = \Gamma(a+x) \Gamma(a-x)
\]
is even, logarithmically convex, strictly increasing on $[0,a)$, satisfies
\[
\lim_{x \nearrow a} F(x) = \lim_{x \searrow -a} F(x) = + \infty
\]
and
\begin{equation}
\label{equation:gamma-property-4}
\forall x \in [0,a) \qquad x \cdot F''(x) \ge F'(x) \; .
\end{equation}
\end{lemma}
\begin{proof}
$F(x)$ is obviously even. $\Gamma(z) = \int_0^\infty t^{z-1} e^{-t} dt$ is
defined for any real number $z > 0$. Hence, $F$ is defined on the interval
$(-a,a)$. According to Bohr-Mollerup theorem \cite[Theorem 2.1]{Artin-1964},
$\Gamma(x)$ is logarithmically convex on $(0,\infty)$. Hence, $F(x)$ is also
logarithmically convex, since $\ln (F(x)) = \ln(\Gamma(a+x)) +
\ln(\Gamma(a-x))$ is a sum of convex functions.

It is well known that $\lim_{z \searrow 0} \Gamma(z) = +\infty$. Thus,
\[
\lim_{x \nearrow a} F(x)
= \lim_{x \nearrow a} \Gamma(a+x) \Gamma(a-x)
= \Gamma(2a) \lim_{x \nearrow a} \Gamma(a-x)
= \Gamma(2a) \lim_{z \searrow 0} \Gamma(z)
= + \infty \; ,
\]
since $\Gamma$ is continuous and not zero at $2a$. Because $F(x)$ is even, we also have
$\lim_{x \searrow -a} F(x) = +\infty$.

To show that $F(x)$ is increasing and that it satisfies
\eqref{equation:gamma-property-4}, we write $f(x) = \ln(F(x))$ as a Mclaurin
series.  The derivatives of $\ln(\Gamma(z))$ are the so called polygamma
functions
\[
\psi^{(n)}(z) = \frac{d^{n+1}}{dz^{n+1}} \ln(\Gamma(z))
\qquad \qquad \text{for $z > 0$ and $n=0,1,2,\dots$.}
\]
Polygamma functions have the well-known integral representation
\[
\psi^{(n)}(z) = (-1)^{n+1} \int_0^\infty \frac{t^n e^{-zt}}{1 - e^{-t}} dt
\qquad \qquad \text{for $z > 0$ and $n=1,2,\dots$.}
\]
Using polygamma functions, we can write the Mclaurin series for $f(x) = \ln(F(x))$ as
\[
f(x) = \ln(F(x)) = \ln(\Gamma(a+x)) + \ln(\Gamma(a-x)) = 2 \ln(\Gamma(a)) + 2 \sum_{\substack{n \ge 2 \\ \text{$n$ even}}} \frac{\psi^{(n-1)}(a) x^n}{n!} \; .
\]
The series converges for $x \in (-a,a)$, since
for even $n \ge 2$, $\psi^{(n-1)}(a)$ is positive and can be upper bounded as
\begin{align*}
\psi^{(n-1)}(a)
& = \int_0^\infty \frac{t^{n-1} e^{-at}}{1 - e^{-t}} dt \\
& = \int_0^1 \frac{t^{n-1} e^{-at}}{1 - e^{-t}} dt + \int_1^\infty \frac{t^{n-1} e^{-zt}}{1 - e^{-t}} dt \\
& \le \int_0^1 \frac{t^{n-1} e^{-at}}{t(1 - 1/e)} dt + \int_1^\infty t^{n-1} e^{-at} dt \\
& \le \frac{1}{1 - 1/e} \int_0^\infty t^{n-2} e^{-at} dt + \int_0^\infty t^{n-1} e^{-at} dt \\
& = \frac{1}{1 - 1/e} a^{1-n} \Gamma(n-1) + a^{-n} \Gamma(n) \\
& \le \frac{1}{1 - 1/e} a^{-n}(a+1) (n-1)!\; .
\end{align*}
From the Mclaurin expansion we see that $f(x)$ is increasing on $[0,a)$
since all the coefficients are positive (except for zero order term).

Finally, to prove \eqref{equation:gamma-property-4}, note that
for any $x \in (-a,a)$,
\[
f(x) = c_0 + \sum_{n=2}^\infty c_n x^n
\]
where $c_2, c_3, \dots$ are non-negative coefficients. Thus
\begin{align*}
f'(x) & = \sum_{n=2}^\infty n c_n x^{n-1} &
& \text{and} &
f''(x) & = \sum_{n=2}^\infty n (n-1) c_n x^{n-2} \; .
\end{align*}
and hence $x \cdot f''(x) \ge f'(x)$ for $x \in [0,a)$. Since $F(x) = \exp(f(x))$,
\begin{align*}
F'(x) & = f'(x) \cdot F(x) &
& \text{and} &
F''(x) & = \left[f''(x) + (f'(x))^2 \right] \cdot F(x) \; .
\end{align*}
Therefore, for $x \in [0,a)$,
\[
x \cdot F''(x)
= x \left[ f''(x) + (f'(x))^2 \right] F(x)
\ge \left[ f'(x) + x (f'(x))^2 \right] F(x)
\ge f'(x) F(x) = F'(x) \; .
\]
This proves \eqref{equation:gamma-property-4}.
\end{proof}

\begin{theorem}[KT potential]
\label{theorem:kt-potential}
Let $\delta \ge 0$ and $\epsilon > 0$. The sequence of functions
$\{F_t\}_{t=0}^\infty$, $F_t:(-t-\delta-1, t+\delta+1) \to \R_+$ defined by
\[
F_t(x) = \epsilon \frac{2^t \cdot \Gamma(\delta + 1) \Gamma(\frac{t+\delta+1}{2} + \frac{x}{2}) \Gamma(\frac{t+\delta+1}{2} - \frac{x}{2})}{\Gamma(\frac{\delta+1}{2})^2 \Gamma(t+\delta+1)} \; .
\]
is a sequence of excellent coin betting potentials for initial endowment $\epsilon$.
Furthermore, for any $x \in (-t-\delta-1, t+\delta+1)$,
\begin{equation}
\label{equation:kt-potential-beta}
\frac{F_t(x+1) - F_{t}(x-1)}{F_t(x+1) + F_{t}(x-1)} = \frac{x}{t+\delta} \; .
\end{equation}
\end{theorem}
\begin{proof}
Property (b) and (d) of the definition follow from
Lemma~\ref{lemma:kt-potential-analytic-properties}.
Property (a) follows by simple substitution for $t=0$ and $x=0$.

Before verifying property (c), we prove \eqref{equation:kt-potential-beta}. We
use an algebraic property of the gamma function that states that $\Gamma(1+z) =
z \Gamma(z)$ for any positive $z$. Equation \eqref{equation:kt-potential-beta}
follows from
\begin{align*}
\frac{F_t(x + 1) - F_t(x - 1)}{F_t(x + 1) + F_t(x - 1)}
& = \frac{\Gamma(\frac{t+\delta+2}{2} + \frac{x}{2}) \Gamma(\frac{t+\delta}{2} - \frac{x}{2}) - \Gamma(\frac{t+\delta}{2} + \frac{x}{2}) \Gamma(\frac{t+\delta+2}{2} - \frac{x}{2})}{\Gamma(\frac{t+\delta+2}{2} + \frac{x}{2}) \Gamma(\frac{t + \delta}{2} - \frac{x}{2}) + \Gamma(\frac{t + \delta}{2} + \frac{x}{2}) \Gamma(\frac{t+\delta+2}{2} - \frac{x}{2})} \\
& = \frac{(\frac{t+\delta}{2} + \frac{x}{2})\Gamma(\frac{t+\delta}{2} + \frac{x}{2}) \Gamma(\frac{t+\delta}{2} - \frac{x}{2}) - (\frac{t+\delta}{2} - \frac{x}{2})\Gamma(\frac{t+\delta}{2} + \frac{x}{2}) \Gamma(\frac{t+\delta}{2} - \frac{x}{2})}{(\frac{t+\delta}{2} + \frac{x}{2})\Gamma(\frac{t+\delta}{2} + \frac{x}{2}) \Gamma(\frac{t+\delta}{2} - \frac{x}{2}) + (\frac{t+\delta}{2} - \frac{x}{2})\Gamma(\frac{t+\delta}{2} + \frac{x}{2}) \Gamma(\frac{t+\delta}{2} - \frac{x}{2})} \\
& = \frac{(\frac{t+\delta}{2} + \frac{x}{2}) - (\frac{t+\delta}{2} - \frac{x}{2})}{(\frac{t+\delta}{2} + \frac{x}{2}) + (\frac{t+\delta}{2} - \frac{x}{2})} \\
& = \frac{x}{t+\delta} \; .
\end{align*}

Let $\phi(g) = \frac{F_t(x+g)}{F_{t-1}(x)}$. To verify property (c) of the definition,
we need to show that $\phi(g) \le 1 + g \frac{x}{t+\delta}$
for any $x \in [-t+1, t-1]$ and any $g \in [-1,1]$. We can write $\phi(g)$ as
\begin{align*}
\phi(g)
& = \frac{F_t(x+g)}{F_{t-1}(x)} \\
& = \frac{2 \Gamma(\frac{t+\delta+1}{2} + \frac{x+g}{2}) \Gamma(\frac{t+\delta+1}{2} - \frac{x+g}{2}) \Gamma(t + \delta)}{\Gamma(\frac{t+\delta}{2} + \frac{x}{2}) \Gamma(\frac{t+\delta}{2} - \frac{x}{2}) \Gamma(t+\delta+1)} \\
& = \frac{2}{t+\delta} \cdot \frac{\Gamma(\frac{t+\delta+1}{2} + \frac{x+g}{2}) \Gamma(\frac{t+\delta+1}{2} - \frac{x+g}{2})}{\Gamma(\frac{t+\delta}{2} + \frac{x}{2}) \Gamma(\frac{t+\delta}{2} - \frac{x}{2})} \; .
\end{align*}
For $g=+1$, using the formula $\Gamma(1+z) = z \Gamma(z)$, we have
\[
\phi(+1)
= \frac{2}{t+\delta} \cdot \frac{\Gamma(\frac{t+\delta}{2} + \frac{x}{2} + 1) \Gamma(\frac{t+\delta}{2} - \frac{x}{2})}{\Gamma(\frac{t+\delta}{2} + \frac{x}{2}) \Gamma(\frac{t+\delta}{2} - \frac{x}{2})}
= \frac{2}{t+\delta} \left(\frac{t+\delta}{2} + \frac{x}{2} \right)
= 1 + \frac{x}{t+\delta} \; .
\]
Similarly, for $g=-1$, using the formula $\Gamma(1+z) = z \Gamma(z)$, we have
\[
\phi(-1)
= \frac{2}{t+\delta} \cdot \frac{\Gamma(\frac{t+\delta}{2} + \frac{x}{2}) \Gamma(\frac{t+\delta}{2} - \frac{x}{2} + 1)}{\Gamma(\frac{t+\delta}{2} + \frac{x}{2}) \Gamma(\frac{t+\delta}{2} - \frac{x}{2})}
= \frac{2}{t+\delta} \left(\frac{t+\delta}{2} - \frac{x}{2} \right)
= 1 - \frac{x}{t+\delta} \; .
\]
We can write any $g \in [-1,1]$ as a convex combination of $-1$ and $+1$, i.e.,
$g = \lambda \cdot (-1) + (1-\lambda) \cdot (+1)$ for some $\lambda \in [0,1]$.
Since $\phi(g)$ is (logarithmically) convex,
\begin{align*}
\phi(g)
& = \phi(\lambda \cdot (-1) + (1-\lambda) \cdot (+1)) \\
& \le \lambda \phi(-1) + (1-\lambda) \phi(+1) \\
& = \lambda \left(1 + \frac{x}{t+\delta}\right) + (1-\lambda) \left(1 - \frac{x}{t+\delta}\right) \\
& = 1 + g \frac{x}{t+\delta} \; . \qedhere
\end{align*}
\end{proof}

\section{Proofs of Corollaries~\ref{corollary:kt-hilbert-space-olo-regret} and~\ref{corollary:kt-experts-regret}}
\label{section:corollaries_reductions}

We state some technical lemmas that will be used in the following proofs. We
start with a lower bound on the \ac{KT} potential. It is a generalization of
the lower bound proved for integers in \citet{Willems-Shtarkov-Tjalkens-1995}
to real numbers.

\begin{lemma}[Lower Bound on KT Potential]
\label{lemma:approx_gamma_real}
If $c \ge 1$ and $a,b$ are non-negative reals such that $a + b = c$ then
\[
\ln \left(\frac{\Gamma(a + 1/2) \cdot \Gamma(b + 1/2)}{\pi \cdot \Gamma(c+1)} \right)
\ge - \ln(e \sqrt{\pi}) -\frac{1}{2} \ln(c) +\ln \left(\left( \frac{a}{c} \right)^a \left( \frac{b}{c} \right)^b\right) \; .
\]
\end{lemma}
\begin{proof}
From \cite{Whittaker-Watson-1962}[p. 263 Ex. 45], we have
\[
\frac{\Gamma(a+1/2)\Gamma(b+1/2)}{\Gamma(a+b+1)} \ge \sqrt{2\pi} \frac{(a+1/2)^a (b+1/2)^b}{(a+b+1)^{a+b+1/2}} \; .
\]
It remains to show that
\[
\sqrt{2\pi} \frac{(a+1/2)^a (b+1/2)^b}{(a+b+1)^{a+b+1/2}} > \frac{\sqrt{\pi}}{e} \frac{1}{\sqrt{a+b}} \left( \frac{a}{a+b} \right)^a \left( \frac{b}{a+b} \right)^b \; ,
\]
which is equivalent to
\[
\frac{(1+\frac{1}{2a})^a (1+\frac{1}{2b})^b}{(1+\frac{1}{a+b})^{a+b+1/2}} > \frac{1}{e\sqrt{2}} \; .
\]
From the inequality $1 \le (1+1/x)^x < e$ valid for any $x \ge 0$, it follows
that $1 \le (1+\frac{1}{2a})^a < \sqrt{e}$ and $1 \le (1+\frac{1}{2b})^b < \sqrt{e}$
and $1 \le (1+1/(a+b))^{a+b} < e$. Hence,
\[
\frac{(1+\frac{1}{2a})^a (1+\frac{1}{2b})^b}{(1+\frac{1}{a+b})^{a+b+1/2}}
> \frac{1}{e \sqrt{1 + \frac{1}{a+b}}}
\ge \frac{1}{e \sqrt{2}} \; . \qedhere
\]
\end{proof}

\begin{lemma}
\label{lemma:ratio_gamma}
Let $\delta \geq 0$. Then
\[
\frac{\Gamma(\delta+1)}{2^\delta \Gamma(\frac{\delta+1}{2})^2}
\geq \frac{\sqrt{\delta+1}}{\pi}~.
\]
\end{lemma}
\begin{proof}
We will prove the equivalent statement that
\[
\ln \frac{\Gamma(\delta+1) \pi}{2^\delta \Gamma(\frac{\delta+1}{2})^2 \sqrt{\delta+1}} \geq 0~.
\]
The inequality holds with equality in $\delta=0$, so it is enough to prove that
the derivative of the left-hand side is positive for $\delta > 0$. The
derivative of the left-hand side is equal to
\[
\Psi(\delta+1) - \frac{1}{2(\delta+1)} -\ln(2) - \Psi\left(\frac{\delta+1}{2} \right) \; ,
\]
where $\Psi(x)$ is the digamma function.

We will use the upper~\citep{Chen-2005} and lower bound~\citep{Batir-2008}
to the digamma function, which state that for any $x>0$,
\begin{align*}
\Psi(x) &< \ln(x) -\frac{1}{2x} -\frac{1}{12 x^2} +\frac{1}{120 x^4} \\
\Psi(x+1) &> \ln \left(x+\frac{1}{2} \right) \; .
\end{align*}
Using these bounds we have
\begin{align*}
&\Psi(\delta+1) - \frac{1}{2(\delta+1)} - \ln(2) - \Psi \left(\frac{\delta+1}{2} \right) \\
&\quad \ge \ln \left(\delta+\frac{1}{2} \right) - \frac{1}{2(\delta+1)} -\ln(2) -\ln \left(\frac{\delta+1}{2}\right) + \frac{1}{\delta+1} +\frac{1}{3 (\delta+1)^2} -\frac{2}{15 (\delta+1)^4}\\
&\quad = \ln \left(1-\frac{1}{2(\delta+1)} \right)+ \frac{1}{2(\delta+1)}+\frac{1}{3 (\delta+1)^2} -\frac{2}{15 (\delta+1)^4}\\
&\quad \ge - \frac{(4\ln(2)-2)}{4(\delta+1)^2}+\frac{1}{3 (\delta+1)^2}-\frac{2}{15 (\delta+1)^4}\\
&\quad = \frac{[15(1/2-\ln(2)))+5](\delta+1)^2-2}{15 (\delta+1)^4}\\
&\quad \ge \frac{[15(1/2-\ln(2)))+5]-2}{15 (\delta+1)^4}\geq0\\
\end{align*}
where in the second inequality we used the elementary inequality $\ln(1-x) \geq -x - (4\ln(2)-2)x^2$ valid for $x \in [0,.5]$.
\end{proof}

\begin{lemma}[Lower Bound on Shifted KT Potential]
\label{lemma:lower_bound_gamma}
Let $T \ge 1$, $\delta \ge 0$, and $x \in [-T,T]$. Then
\begin{align*}
\frac{2^T \cdot \Gamma(\delta+1) \Gamma \left(\frac{T+\delta+1}{2} + \frac{x}{2} \right) \cdot \Gamma \left(\frac{T+\delta+1}{2} - \frac{x}{2} \right)}{ \Gamma(\frac{\delta+1}{2})^2 \Gamma(T+\delta+1)}
&\geq\exp\left(\frac{x^2}{2(T+\delta)} +\frac{1}{2} \ln \left(\frac{1+\delta}{T+\delta}\right) - \ln(e \sqrt{\pi})\right) \;.
\end{align*}
\end{lemma}
\begin{proof}
Using Lemma~\ref{lemma:approx_gamma_real}, we have
\begin{align*}
&\ln \frac{2^T \cdot \Gamma(\delta+1) \Gamma \left(\frac{T+\delta+1}{2} + \frac{x}{2} \right) \cdot \Gamma \left(\frac{T+\delta+1}{2} - \frac{x}{2} \right)}{ \Gamma(\frac{\delta+1}{2})^2 \Gamma(T+\delta+1)} \\
&\quad \geq \ln \frac{2^{T+\delta} \sqrt{\delta+1} \cdot \Gamma \left(\frac{T+\delta+1}{2} + \frac{x}{2} \right) \cdot \Gamma \left(\frac{T+\delta+1}{2} - \frac{x}{2} \right)}{ \pi \Gamma(T+\delta+1)} \\
&\quad \geq -\ln(e \sqrt{\pi}) +\frac{1}{2} \ln\left(\frac{1+\delta}{T+\delta}\right) +\ln \left(\left( 1+\frac{x}{T+\delta} \right)^\frac{T+\delta+x}{2} \left( 1+\frac{x}{T+\delta} \right)^\frac{T+\delta-x}{2}\right) \\
&\quad = -\ln(e \sqrt{\pi}) +\frac{1}{2} \ln\left(\frac{1+\delta}{T+\delta}\right)+ (T+\delta) \, \KL{\frac{1}{2}+\frac{x}{2(T+\delta)}}{\frac{1}{2}} \\
&\quad \geq -\ln(e \sqrt{\pi}) +\frac{1}{2} \ln\left(\frac{1+\delta}{T+\delta}\right)+ \frac{x^2}{2(T+\delta)} ,
\end{align*}
where in the first inequality we used Lemma~\ref{lemma:ratio_gamma}, in the second one Lemma~\ref{lemma:approx_gamma_real}, and in third one the known lower bound to the divergence $\KL{\frac{1}{2}+\frac{x}{2}}{\frac{1}{2}} \geq \frac{x^2}{2}$. Exponentiating and overapproximating, we get the stated bound.
\end{proof}

\subsection{Proof of Corollary~\ref{corollary:kt-hilbert-space-olo-regret}}

The Lambert function $W(x):[0,\infty) \to [0,\infty)$ is defined by the equality
\begin{equation}
\label{eq:lambert}
x=W(x) \exp \left(W(x)\right) \qquad \qquad \text{for $x \ge 0$}.
\end{equation}
The following lemma provides bounds on $W(x)$.

\begin{lemma}
\label{lemma:lambert}
The Lambert function satisfies $0.6321 \log(x+1) \leq W(x) \leq \log(x+1)$ for $x \ge 0$.
\end{lemma}
\begin{proof}
The inequalities are satisfied for $x=0$, hence we in the following we assume $x>0$.
We first prove the lower bound. From \eqref{eq:lambert} we have
\begin{equation}
W(x) = \log\left(\frac{x}{W(x)}\right)~. \label{eq:lm_lambert_1}
\end{equation}
From the first equality, using the elementary inequality $\ln (x) \leq \frac{a}{e} x^\frac{1}{a}$ for any $a>0$, we get
\[
W(x) \leq \frac{1}{a\, e}\left(\frac{x}{W(x)}\right)^a  \ \ \forall a>0,
\]
that is
\begin{equation}
\label{eq:lm_lambert_2}
W(x) \leq \left(\frac{1}{a\, e}\right)^\frac{1}{1+a} x^\frac{a}{1+a} \ \ \forall a>0.
\end{equation}
Using \eqref{eq:lm_lambert_2} in \eqref{eq:lm_lambert_1}, we have
\begin{align*}
W(x)
\geq \log\left(\frac{x}{\left(\frac{1}{a\, e}\right)^\frac{1}{1+a} x^\frac{a}{1+a}}\right)
= \frac{1}{1+a}\log\left(a \, e\, x\right) \ \ \forall a>0~.
\end{align*}
Consider now the function $g(x)=\frac{x}{x+1} - \frac{b}{\log(1+b) (b+1)}
\log(x+1), x\geq b$. This function has a maximum in $x^*=(1+\frac{1}{b})
\log(1+b)-1$, the derivative is positive in $[0,x^*]$ and negative in
$[x^*,b]$. Hence the minimum is in $x=0$ and in $x=b$, where it is equal to
$0$.  Using the property just proved on $g$, setting $a=\frac{1}{x}$, we have
\begin{align*}
W(x)
\geq \frac{x}{x+1} \geq \frac{b}{\log(1+b) (b+1)} \log(x+1) \ \  \forall x\leq b~.
\end{align*}
For $x>b$, setting $a=\frac{x+1}{e x}$, we have
\begin{align}
W(x)
&\geq \frac{e\,x}{(e+1) x + 1} \log(x+1) \geq \frac{e\,b}{(e+1) b + 1} \log(x+1)
\end{align}
Hence, we set $b$ such that
\[
\frac{e\, b}{(e+1)b + 1} = \frac{b}{\log(1+b) (b+1)}
\]
Numerically, $b=1.71825...$, so
\[
W(x) \geq 0.6321 \log(x+1)~.
\]

For the upper bound, we use Theorem~2.3 in \cite{Hoorfar-Hassani-2008}, that says that
\[
W(x) \leq \log\frac{x+C}{1+\log(C)}, \quad \forall x> -\frac{1}{e}, \ C>\frac{1}{e}.
\]
Setting $C=1$, we obtain the stated bound.
\end{proof}

\begin{lemma}
\label{lemma:dual_exp_square}
Define $f(x)= \beta \exp\frac{x^2}{2 \alpha}$, for $\alpha,\beta>0$, $x\geq0$. Then
\[
f^*(y)=y \sqrt{\alpha W\left(\frac{\alpha y^2}{\beta^2}\right)} - \beta \exp\left(\frac{W\left(\frac{\alpha y^2}{\beta^2}\right)}{2}\right).
\]
Moreover
\[
f^*(y) \leq y \sqrt{\alpha \log \left(\frac{\alpha y^2}{\beta^2} +1 \right)} - \beta.
\]
\end{lemma}
\begin{proof}
From the definition of Fenchel dual, we have
\begin{align*}
f^*(y)= \max_{x} \  x\, y - f(x) = \max_{x} \  x\, y - \beta \exp\frac{x^2}{2 \alpha} \leq x^*\,y -\beta
\end{align*}
where $x^*= \argmax_{x} x\, y - f(x)$. We now use the fact that $x^*$ satisfies $y = f'(x^*)$, to have
\begin{align*}
x^*=\sqrt{\alpha W\left(\frac{\alpha y^2}{\beta^2}\right)},
\end{align*}
where $W(\cdot)$ is the Lambert function.
Using Lemma~\ref{lemma:lambert}, we obtain the stated bound.
\end{proof}

\begin{proof}[Proof of Corollary~\ref{corollary:kt-hilbert-space-olo-regret}]
Notice that the KT potential can be written as
\[
F_t(x) = \epsilon \cdot \frac{2^t \cdot \Gamma(1) \Gamma \left(\frac{t+1}{2} + \frac{x}{2} \right) \cdot \Gamma \left(\frac{t+1}{2} - \frac{x}{2} \right)}{ \Gamma(\frac{1}{2})^2 \Gamma(t+1)} \; .
\]
Using Lemma~\ref{lemma:lower_bound_gamma} with $\delta = 0$ we can lower bound $F_t(x)$ with
\[
H_t(x) = \epsilon \cdot \exp\left(\frac{x^2}{2t} + \frac{1}{2} \ln \left(\frac{1}{t} \right) - \ln (e \sqrt{\pi}) \right) \; .
\]
Since $H_t(x) \le F_t(x)$, we have $F^*_t(x) \le H_t^*(x)$. Using Lemma~\ref{lemma:dual_exp_square}, we have
\[
\forall u \in \H \qquad \qquad
F^*_T\left(\norm{u}\right)
\le H^*_T\left(\norm{u}\right)
\le \sqrt{T \log \left(\frac{24 T^2 \norm{u}^2}{\epsilon^2} +1 \right)}+\epsilon\left(1-\frac{1}{e \sqrt{ \pi T}}\right).
\]
An application of Theorem~\ref{theorem:hilbert-space-olo-regret-bound} completes the proof.
\end{proof}

\subsection{Proof of Corollary~\ref{corollary:kt-experts-regret}}

\begin{proof}
Let
\begin{align*}
F_t(x) & = \frac{2^t \cdot \Gamma(\delta + 1) \Gamma(\frac{t+\delta+1}{2} + \frac{x}{2}) \Gamma(\frac{t+\delta+1}{2} - \frac{x}{2})}{\Gamma(\frac{\delta+1}{2})^2 \Gamma(t+\delta+1)} \; , \\
H_t(x) & = \exp\left(\frac{x^2}{2(t+\delta)} +\frac{1}{2} \ln \left(\frac{1+\delta}{t+\delta}\right) - \ln(e \sqrt{\pi})\right) \; .
\end{align*}
Let $f_t(x) = \ln(F_t(x))$ and $h_t(x)=\ln(H_t(x))$. By Lemma~\ref{lemma:lower_bound_gamma}, $H_t(x) \le F_t(x)$
and therefore $f^{-1}_t(x) \leq h^{-1}_t(x)$ for all $x \ge 0$.
Theorem~\ref{theorem:regret-bound-experts} implies that
\[
\forall u \in \Delta_t \qquad \qquad
\Regret_t(u) \le f_t^{-1}(\KL{u}{\pi}) \le h_t^{-1}(\KL{u}{\pi}) \; .
\]
Setting $t=T$ and $\delta = T/2$, and overapproximating $h_t^{-1}(\KL{u}{\pi})$ we get the stated bound.
\end{proof}

\end{document}